\documentclass{article}
\usepackage{iclr2027_conference,times}

\usepackage{amsmath,amsfonts,bm}

\def\eqref#1{equation~\ref{#1}}

\def\1{\bm{1}}

\def\vtheta{{\bm{\theta}}}

\def\vg{{\bm{g}}}

\def\vr{{\bm{r}}}
\def\vs{{\bm{s}}}

\def\vw{{\bm{w}}}
\def\vx{{\bm{x}}}

\def\vz{{\bm{z}}}

\def\mA{{\bm{A}}}

\def\mD{{\bm{D}}}

\def\mH{{\bm{H}}}
\def\mI{{\bm{I}}}

\def\mL{{\bm{L}}}
\def\mM{{\bm{M}}}

\def\mPhi{{\bm{\Phi}}}

\DeclareMathAlphabet{\mathsfit}{\encodingdefault}{\sfdefault}{m}{sl}
\SetMathAlphabet{\mathsfit}{bold}{\encodingdefault}{\sfdefault}{bx}{n}

\usepackage{url}
\usepackage{amsmath,amssymb,amsthm}
\usepackage{algorithm}
\usepackage{algorithmic}
\usepackage{graphicx}
\usepackage{booktabs}
\usepackage{tabularx}
\usepackage{multirow}
\usepackage{xcolor}
\usepackage{capt-of}
\usepackage{enumitem}
\usepackage{hyperref}
\setlist[itemize]{leftmargin=*}

\newtheorem{theorem}{Theorem}
\newtheorem{lemma}{Lemma}

\newtheorem{corollary}{Corollary}
\theoremstyle{definition}
\newtheorem{definition}{Definition}
\newtheorem{assumption}{Assumption}

\theoremstyle{remark}
\newtheorem{remark}{Remark}

\DeclareMathOperator{\diag}{diag}
\newcommand{\RR}{\mathbb{R}}
\newcommand{\cL}{\mathcal{L}}
\newcommand{\cC}{\mathcal{C}}
\newcommand{\cO}{\mathcal{O}}
\newcommand{\cS}{\mathcal{S}}
\newcommand{\vtg}{\tilde{\mathbf{g}}}

\title{Predicting Block-Coordinate Performance via Cross-Curvature}

\author{
Shengkun Zhu\textsuperscript{1} \quad
Jinshan Zeng\textsuperscript{2} \quad
Zhiqiang Kou\textsuperscript{1} \quad
Yongxin Tong\textsuperscript{3} \quad
Yang Liu\textsuperscript{1}
\\[4pt]
\textsuperscript{1}The Hong Kong Polytechnic University\quad
\textsuperscript{2}School of Management, Xi'an Jiaotong University\\
\textsuperscript{3}Beijing Key Laboratory of AI-Native Data Systems and SKLCCSE Lab, Beihang University
\\[4pt]
{\small\texttt{\{shengkun96.zhu,zhiqikou,yang-veronica.liu\}@polyu.edu.hk,
}}\\
{\small\texttt{jsh.zeng@gmail.com,
yxtong@buaa.edu.cn
}}
}

\renewcommand{\eqref}[1]{Equation~(\ref{#1})}

\newcommand{\eqrefs}[2]{Equations~(\ref{#1}) and~(\ref{#2})}
\newcommand{\eqrefsiii}[3]{Equations~(\ref{#1}), (\ref{#2}) and~(\ref{#3})}
  \iclrfinalcopy
\begin{document}

\addtocontents{toc}{\protect\setcounter{tocdepth}{-1}}

\maketitle
 \lhead{Preprint} 
\suppressfloats[t]

\begin{abstract}
Simultaneous and sequential block updates are two basic optimization strategies used across machine learning, such as neural-network training, federated learning, and low-rank adaptation.
Choosing between them is difficult because their relative advantage depends on both the objective geometry and the number of iterations.
We develop a unified theory for comparing Jacobi (JC), Gauss--Seidel (GS), and partially sequential deterministic block-gradient updates.
Our analysis expresses the one-step loss difference through cross-block curvature, with an $\cO(\eta^3)$ remainder, where $\eta$ is the learning rate.
We derive a signed loss comparison after $K$ iterations with $\cO(K\eta^3)$ error under regularity conditions and $\eta K\le T$ for fixed $T$, identifying the better method when the predicted difference exceeds this error.
We evaluate these formulas along observed training trajectories across
different machine learning settings. Over 500 iterations, our theory correctly identifies the lower-loss method in 98.0\% of iterations for the neural network, 83.4\% for federated learning, and 97.6\% for LoRA. Applying the loss recursion at each step using the measured parameter difference raises these rates to 100.0\%, 93.2\%, and 99.6\%, respectively.
\end{abstract}

\section{Introduction}

Block-coordinate optimization is widely used in machine learning. Across deep learning, federated learning, large model training and fine-tuning, and many other domains. High-dimensional parameters $\vtheta \in \RR^d$ are typically partitioned into blocks $(\vtheta_1, \ldots, \vtheta_N)$ that are updated iteratively. 
A fundamental question arises: should blocks update simultaneously or sequentially? 
This choice gives rise to two classical strategies. 
In \emph{Jacobi} iteration, all blocks update in parallel using gradients evaluated at the current shared state, which maximizes parallelism but relies on outdated information as other blocks change. 
In \emph{Gauss-Seidel} iteration, blocks update sequentially, with each block incorporating the latest values from previously updated blocks; this approach sacrifices parallelism but benefits from fresher gradient information. 
The choice between these strategies, or hybrid schemes mixing both, profoundly affects convergence, yet remains guided by intuition and empirical experience rather than principled theory.

\begin{figure}[t]
\centering
\includegraphics[width=\textwidth]{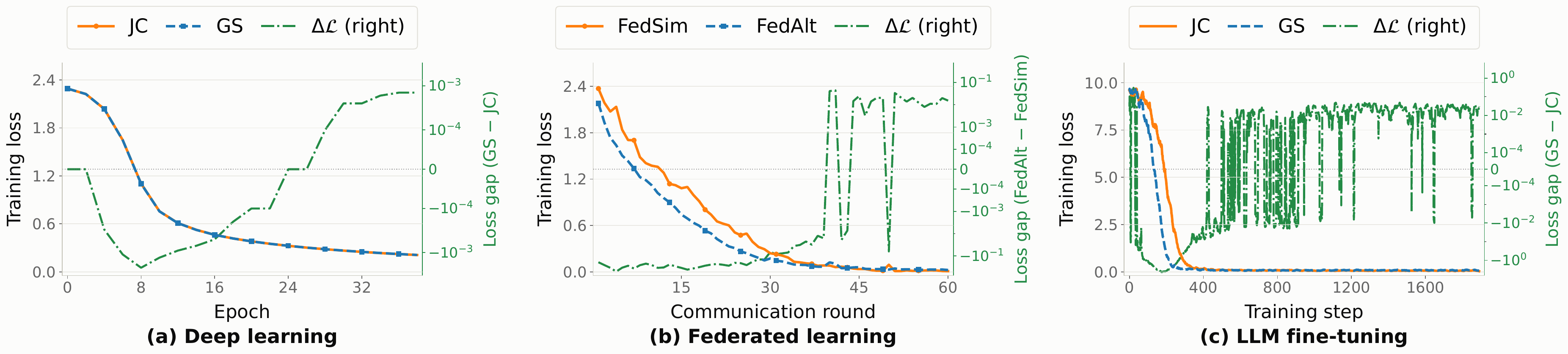}
\vspace{-1em}
\caption{Training loss for (a) an 8-layer MLP on MNIST, (b) federated ResNet-18 on CIFAR-10, and (c) T5-base LoRA fine-tuning on SST-2.
Blue and orange curves denote sequential/alternating (GS, FedAlt) and simultaneous (JC, FedSim) updates, respectively.
Green dash-dotted curves (right axes) show $\Delta\cL=\cL_{\mathrm{GS}}-\cL_{\mathrm{JC}}$ (FedAlt minus FedSim in panel~(b)).}
\label{fig:motivation}
\end{figure}

Despite the widespread use of block-coordinate methods, \emph{we lack a unified theory predicting which update scheme performs better under what conditions}. Figure~\ref{fig:motivation} illustrates this challenge in three settings (see Details in Appendix \ref{app:intro_methods}): (a) an 8-layer MLP on MNIST with layer-wise simultaneous versus GS-like in-place backward updates; (b) federated ResNet-18 on CIFAR-10 with FedAlt (alternating personal/shared updates, GS-like) versus FedSim (simultaneous updates, JC-like); and (c) T5-base LoRA fine-tuning on SST-2 with alternating versus simultaneous updates of the $A$ and $B$ low-rank factors. Across these settings, the recorded training losses favor sequential updates early in training and simultaneous updates later. This reversal motivates a criterion for predicting when and under what conditions the relative advantage changes.

Classical results from numerical linear algebra \citep{young1971iterative,saad2003iterative} analyze Gauss-Seidel versus Jacobi for solving linear systems $Ax = b$, proving that Gauss-Seidel converges faster for symmetric positive definite matrices. However, these results do not extend to nonlinear objectives $\min_{\vtheta} \cL(\vtheta)$, nor do they predict the instance-specific performance gaps visible in Figure~\ref{fig:motivation}. Modern optimization theory \citep{nesterov2012efficiency,attouch2013convergence, beck2013convergence, wright2015coordinate, wang2019global} establishes convergence rates for various block-coordinate descent variants, but focuses on asymptotic guarantees that obscure the fine-grained question: \emph{for a given problem at a given iterate, which update scheme makes more progress, and by how much?} Asynchronous optimization \citep{lian2015asynchronous} and federated learning \citep{pillutla2022fedalt} literature documents empirical benefits of fresher parameter values but lacks geometric characterizations of when staleness helps versus hurts. The gap between classical linear-system theory and modern nonlinear optimization leaves practitioners without actionable guidance for choosing update schemes in applications.

We develop a curvature-based theory of \emph{iteration selection}: given an objective, a block partition, and an iteration budget, which update scheme should be used to obtain a lower objective value? Our framework compares simultaneous, sequential, and partially sequential block-gradient updates through the interactions between block gradients and cross-block curvature. The resulting criterion is state-dependent: the preferred scheme can change as optimization progresses. Over multiple iterations, the loss gap also depends on how earlier updates have changed the methods' parameter trajectories. We account for this history through trajectory and loss recursions, yielding a cumulative comparison after a prescribed number of iterations.
Our contributions are summarized as follows:
\begin{itemize}
\item We derive a unified one-step loss comparison for acyclic block-gradient schemes (Theorem~\ref{thm:master}). From a common starting point, the GS--JC gap is $\frac{\eta^2}{2}\cC_{\mathrm{cross}}+\cO(\eta^3)$, where $\cC_{\mathrm{cross}}=2\sum_{i<j}\vg_i^\top\mH_{ij}\vg_j$. The sign of cross-curvature identifies the lower-loss one-step update whenever the leading term dominates the remainder.

\item We derive trajectory and loss recursions with $\cO(\eta^3)$ local remainders (Lemma~\ref{lem:multistep_recursion}; Theorem~\ref{thm:loss_recursion}). Under the stated regularity conditions and $\eta K\le T$ for fixed $T$, their propagation gives a signed terminal loss comparison with $\cO(K\eta^3)$ error (Theorem~\ref{thm:cumulative}), identifying the lower-loss method when the estimated gap magnitude exceeds the remainder bound.

\item We evaluate recursive and cumulative estimates using observed trajectory derivatives in DNN, federated learning, and LoRA experiments (Section~\ref{sec:experiments}). Over 500 iterations per setting, cumulative winner agreement is $98.0\%$, $83.4\%$, and $97.6\%$, respectively. The recursive estimate improves winner agreement and loss-gap accuracy in all three settings.
\end{itemize}

\section{Related Work}

\paragraph{Coordinate descent methods.}
Coordinate descent exploits a block partition to reduce the cost of
individual updates. Its theory covers cyclic, randomized, and greedy
block selection, as well as smooth and composite objectives
\citep{wright2015coordinate}.
\citet{tseng2001convergence} establishes convergence results for
nondifferentiable block-coordinate minimization, while
\citet{nesterov2012efficiency,richtarik2014iteration} derive
non-asymptotic complexity bounds for randomized coordinate methods.
Block successive-minimization frameworks further accommodate inexact
updates through local surrogate subproblems
\citep{razaviyayn2013unified}.
These results characterize how objective structure, block selection,
and subproblem accuracy affect convergence.

\paragraph{Theory of iterative methods.}
Classical Jacobi and Gauss--Seidel analyses use matrix splittings and
iteration spectra to study stability and convergence rates
\citep{young1971iterative,saad2003iterative}.
Comparisons depend on the matrix structure and the update definition;
exact block solves and block-gradient steps induce different iteration
operators.
Asynchronous optimization extends convergence analysis to delayed or
inconsistent information, relating admissible parallelism to sparsity,
smoothness, and delay assumptions
\citep{recht2011hogwild,lian2015asynchronous}.
Beyond convergence bounds, \citet{rosca2021drift} use backward error
analysis to characterize the distinct discretization effects of
simultaneous and alternating updates in two-player games.
Initialization-based approximations offer another view of training
dynamics: \citet{lee2019wide} show that, in the infinite-width limit,
neural-network evolution is governed by a linearization of the network
output around initialization.

\paragraph{Applications in machine learning.}
Coupled parameter blocks arise throughout learning.
Back-propagation computes gradients for layer-structured neural networks
\citep{lecun1988framework}.
In federated learning \citep{mcmahan2017communication}, shared representations interact with
client-specific parameters: \citet{pillutla2022fedalt} analyze
simultaneous FedSim and alternating FedAlt updates under nonconvex
objectives and partial participation, identifying regimes favorable to
alternation.
Low-rank adaptation introduces a different coupling through
$\Delta W=BA$ \citep{hu2021lora}.
LoRA+ studies unequal learning rates for the two factors
\citep{hayou2024loraplus}, whereas LoRA-E$^2$ combines a modified
initialization with $B$-first GS warm-up followed by simultaneous
updates \citep{zhu2026lorae2}.
These applications illustrate how parameter coupling motivates choices
of update order, information freshness, and factor-specific step sizes.
Beyond machine learning, the theory of iterative methods also plays a central role in optimization \citep{richtarik2016parallel, xu2017globally},
signal processing \citep{kerahroodi2017coordinate}, and computational imaging \citep{wang2008new, xu2013block}.

Taken together, these studies establish convergence guarantees,
characterize iterative dynamics, and develop application-specific
update rules. Our focus is the complementary problem of
\emph{instance-specific iteration selection}: comparing the signed
terminal losses of fixed JC and GS methods at a prescribed budget.
For deterministic block-gradient updates, cross-curvature yields a
one-step comparison with $\cO(\eta^3)$ remainder, and signed trajectory
propagation extends the comparison to a prescribed number of iterations.
This distinguishes local update preference from comparison after multiple iterations without asserting universal superiority of either scheme.

\section{Preliminaries}


\subsection{Block-Coordinate Optimization}
We use bold lowercase letters for vectors (e.g., $\vtheta, \vg$) and bold uppercase letters for matrices (e.g., $\mH, \mA$). For vectors $\vx \in \RR^d$, $\|\vx\|$ denotes the Euclidean norm. For matrices $\mA \in \RR^{m \times n}$, $\|\mA\|$ denotes the spectral norm (largest singular value). We write $[N] = \{1, 2, \ldots, N\}$ for index sets. 
We consider the optimization problem $\min_{\vtheta \in \RR^d} \cL(\vtheta)$ where $\cL: \RR^d \to \RR$ is twice continuously differentiable with a locally Lipschitz Hessian on a neighborhood containing the updates and their intermediate states. The parameter vector $\vtheta$ is partitioned into $N$ blocks: $\vtheta = (\vtheta_1, \ldots, \vtheta_N)$ with $\vtheta_i \in \RR^{d_i}$ and $\sum_{i=1}^N d_i = d$. At iterate $\vtheta^k$, we denote the gradient $\vg = \nabla \cL(\vtheta^k) \in \RR^d$ with block components $\vg_i = \nabla_{\vtheta_i} \cL(\vtheta^k) \in \RR^{d_i}$, and the Hessian $\mH = \nabla^2 \cL(\vtheta^k) \in \RR^{d \times d}$ with blocks $\mH_{ij} = \nabla_{\vtheta_i, \vtheta_j}^2 \cL(\vtheta^k) \in \RR^{d_i \times d_j}$ satisfying $\mH_{ij} = \mH_{ji}^\top$ by Schwarz's theorem. We decompose the Hessian as $\mH = \mD + \mL + \mL^\top$ where $\mD = \diag(\mH_{11}, \ldots, \mH_{NN})$ is block-diagonal and $\mL$ is strictly lower triangular with $(\mL)_{ij} = \mH_{ij}$ for $i > j$ and zero otherwise.

\begin{definition}[Jacobi and Gauss-Seidel Updates]
\label{def:updates}
Given learning rate $\eta > 0$ and current iterate $\vtheta^k$, we define two fundamental update schemes:
\begin{itemize}
\item \textbf{Jacobi (JC):} All blocks update in parallel using gradients computed at $\vtheta^k$: $\vtheta_i^{k+1} = \vtheta_i^k - \eta \vg_i^k$ for all $i \in [N]$, where $\vg_i^k = \nabla_{\vtheta_i} \cL(\vtheta^k)$. All blocks use the same parameter state, enabling parallel computation.

\item \textbf{Gauss-Seidel (GS):} Blocks update sequentially in order $i = 1, 2, \ldots, N$, each using the freshest available values: $\vtheta_i^{k+1} = \vtheta_i^k - \eta \vtg_i^k$, where the fresh gradient $\vtg_i^k = \nabla_{\vtheta_i} \cL(\tilde{\vtheta}^{(i)})$ is computed at the partially updated state $\tilde{\vtheta}^{(i)} = (\vtheta_1^{k+1}, \ldots, \vtheta_{i-1}^{k+1}, \vtheta_i^k, \ldots, \vtheta_N^k)$. Block $i$ uses new values $\vtheta_j^{k+1}$ for all $j < i$, creating sequential dependencies.
\end{itemize}
\end{definition}

\begin{remark}[Geometric interpretation]
The difference between JC and GS is the parameter state used to evaluate each block gradient. For GS in the order $1,\ldots,N$, a first-order Taylor expansion gives
\begin{align*}
\vtg_i^k
= \vg_i^k + \sum_{j<i}\mH_{ij}(\vtheta_j^{k+1}-\vtheta_j^k) + \cO(\eta^2)
= \vg_i^k - \eta\sum_{j<i}\mH_{ij}\vg_j^k + \cO(\eta^2).
\end{align*}
Thus, each block update $-\eta\vg_j^k$ changes the gradient of block $i$ through $\mH_{ij}$. In vector form, the GS gradient is $\vtg^k = \vg^k - \eta\mL\vg^k + \cO(\eta^2)$, whereas JC uses $\vg^k$ without this correction.
\end{remark}

\subsection{General Update Schemes via Delay Patterns}

Jacobi and Gauss-Seidel represent two extreme strategies: JC uses completely outdated parameters (all blocks see $\vtheta^k$), while full GS uses maximally fresh parameters (block $i$ sees all updates $\vtheta_j^{k+1}$ for $j < i$). However, practical distributed systems often operate between these extremes: some blocks may share fresh values while others remain outdated due to communication delays, asynchronous execution, or deliberate scheduling. To analyze such hybrid schemes that encompass JC, GS, and everything in between, we introduce the delay pattern.
\begin{definition}[Delay Pattern]
\label{def:staleness}
A \emph{delay pattern} is a set $\cS \subseteq \{(i,j) : i,j \in [N]\}$, where $(i,j)\in\cS$ means that block $i$ uses the fresh value $\vtheta_j^{k+1}$ rather than $\vtheta_j^k$ when computing its update.
We call $\cS$ \emph{acyclic} if the blocks can be ordered so that, for every $(i,j)\in\cS$, block $j$ is updated before block $i$.
\end{definition}

To illustrate the expressiveness of this formalism, consider $N = 3$ blocks: Jacobi corresponds to $\cS = \emptyset$ (no fresh values), while full GS in order $1 \to 2 \to 3$ gives $\cS = \{(2,1), (3,1), (3,2)\}$ (each block uses all preceding updates). Between these extremes lie hybrid schemes such as partial GS with $\cS = \{(2,1)\}$ (only block 2 uses fresh $\vtheta_1$) or a chain pattern $\cS = \{(2,1), (3,2)\}$ (block 2 uses fresh $\vtheta_1$, block 3 uses fresh $\vtheta_2$). The key insight is that any acyclic pattern $\cS$ admits a first-order gradient correction, formalized in the following lemma.

\begin{lemma}[Gradient Approximation]
\label{lem:staleness}
For acyclic $\cS$ and small $\eta$, the fresh gradient satisfies:
$$\vtg_i = \vg_i - \eta \sum_{j : (i,j) \in \cS} \mH_{ij} \vg_j + \cO(\eta^2).$$
In vector form, $\vtg^\cS = \vg - \eta \mM_\cS \vg + \cO(\eta^2)$, where $(\mM_\cS)_{ij} = \mH_{ij}$ if $(i,j) \in \cS$ and zero otherwise.
\end{lemma}

This lemma separates the update schedule from the objective geometry:
$\cS$ selects the fresh-value dependencies, while the retained Hessian determine how those updates change each block gradient. JC has
$\mM_\cS=0$, full GS in the order $1,\ldots,N$ has $\mM_\cS=\mL$.
The proof is given in Appendix~\ref{app:gradient_delay}.

\section{The Main Theorem}
\label{sec:main_theorem}
All proofs and intermediate derivations are collected in Appendices~\ref{app:single_step}--\ref{app:cumulative}. The main text states the results and their interpretation.

\subsection{Main Result}

We now leverage the delay pattern formalism and the gradient approximation from Lemma \ref{lem:staleness} to derive a unified second-order expansion that applies to any acyclic update scheme. This expansion reveals how the delay pattern $\cS$ induces an extra curvature term $\cC_\cS$ beyond the standard descent, quantifying the performance gap between different update strategies.

\begin{theorem}[Main Theorem]
\label{thm:master}
Let $\cS$ be acyclic, $\eta > 0$ sufficiently small, and consider both update schemes starting from the same iterate $\vtheta^k$ with gradient $\vg = \nabla \cL(\vtheta^k)$ and Hessian $\mH = \nabla^2 \cL(\vtheta^k)$. Then after one step:
\begin{align}
\cL_{\text{JC}}^{k+1} &= \cL^k - \eta \|\vg\|^2 + \frac{\eta^2}{2} \vg^\top \mH \vg + \cO(\eta^3), \\
\cL_{\cS}^{k+1} &= \cL^k - \eta \|\vg\|^2 + \frac{\eta^2}{2} \vg^\top \mH \vg + \eta^2 \cC_\cS + \cO(\eta^3),
\end{align}
where $\cC_\cS = \vg^\top \mM_\cS \vg = \sum_{(i,j) \in \cS} c_{ij}$ and $c_{ij} = \vg_i^\top \mH_{ij} \vg_j$. Therefore:
$$\boxed{\cL_\cS^{k+1} - \cL_{\text{JC}}^{k+1} = \eta^2 \cC_\cS + \cO(\eta^3).}$$
\end{theorem}

The theorem isolates the effect of the update scheme: all schemes share
the same first-order decrease, while the use of fresh block values
contributes the second-order correction $\eta^2\cC_\cS$.
Each term $c_{ij}$ measures the signed interaction between two block
gradients through their cross-Hessian block.
Consequently, $\cC_\cS<0$ favors scheme $\cS$ over JC, whereas
$\cC_\cS>0$ favors JC, provided the leading term dominates the remainder.
This gives a local selection criterion for one complete block sweep
from a common starting point.
The proof is given in Appendix~\ref{app:main_theorem}.

For full GS, let $\cS_{\mathrm{GS}}=\{(i,j):i>j\}$ and define the cross-curvature by
\begin{equation}
\cC_{\text{cross}}
:=2\sum_{i<j}\vg_i^\top\mH_{ij}\vg_j
=\vg^\top(\mH-\mD)\vg.
\label{eq:cross_curvature}
\end{equation}
The following corollary states the resulting comparison with JC; its derivation and the update-order argument are given in Appendix~\ref{app:gs_criterion}.

\begin{corollary}[GS vs. JC Criterion]
\label{cor:criterion}
Applying the Main Theorem to the specific case of full Gauss-Seidel versus Jacobi, we obtain:
$$\cL_{\text{GS}}^{k+1} - \cL_{\text{JC}}^{k+1} = \frac{\eta^2}{2} \cC_{\text{cross}} + \cO(\eta^3).$$
When the displayed leading term dominates the remainder, this gives the following sign criterion:
\begin{itemize}
\item $\cC_{\text{cross}} < 0 \implies$ \textbf{GS wins} (achieves lower loss than JC). Negative cross-curvature indicates that off-diagonal Hessian blocks couple the gradients cooperatively, meaning fresh parameter updates in earlier blocks help reduce the loss in later blocks.
\item $\cC_{\text{cross}} > 0 \implies$ \textbf{JC wins} (achieves lower loss than GS). Positive cross-curvature indicates that block gradients compete through off-diagonal couplings, and using fresh updates actually increases the loss compared to using stale but consistent gradients.
\item $\cC_{\text{cross}} = 0 \implies$ \textbf{No difference} to $\cO(\eta^2)$. The two methods are equivalent at the leading order, with differences appearing only at $\cO(\eta^3)$ or higher.
\end{itemize}
\end{corollary}

\subsection{Multi-Step Dynamics}
At the common initialization $\vtheta^0$, the one-step comparison is
$\cL_{\text{GS}}^1 - \cL_{\text{JC}}^1 = \eta^2 \cC_{\cS_{\mathrm{GS}}}(\vtheta^0) + \cO(\eta^3).$
After step 0, the trajectories diverge: $\vtheta_{\text{GS}}^k \neq \vtheta_{\text{JC}}^k$. The Main Theorem no longer applies because it requires both methods to start from the same point. Instead, we track the loss difference through the coupled pair of \eqrefs{eq:traj_evolution}{eq:loss_evolution} stated in Lemma~\ref{lem:multistep_recursion} and Theorem~\ref{thm:loss_recursion}, in which the trajectory difference is the state variable and the loss difference is read off from it at each step (see Appendix~\ref{app:trajectory_recursion}).

\subsubsection{Standing assumptions and notation}

To compare the two methods over $K$ iterations, we need to control how their iterates separate and how approximation errors accumulate. The following assumptions bound the gradient, Hessian, and Hessian variation throughout both sequences of iterates and the intermediate block updates. We assume $\eta K\le T$, where $T>0$ is fixed independently of $\eta$ and $K$.

\begin{assumption}
\label{ass:regularity}
There is an open convex set $\Omega$ containing all block-intermediate states, both trajectories $\{\vtheta_\cS^k\}_{k \le K}$ and $\{\vtheta_{\text{JC}}^k\}_{k \le K}$ together with the segments joining $\vtheta_{\text{JC}}^k$ to $\vtheta_\cS^k$, on which $\cL$ is three times continuously differentiable with
\begin{enumerate}
\item[\textbf{(A1)}] $\|\nabla\cL(\vtheta)\| \le G$ \hfill (bounded gradients)
\item[\textbf{(A2)}] $\|\nabla^2\cL(\vtheta)\| \le L$ \hfill ($L$-smoothness)
\item[\textbf{(A3)}] $\|\nabla^2\cL(\vtheta) - \nabla^2\cL(\vtheta')\| \le \rho\,\|\vtheta - \vtheta'\|$ \hfill (Lipschitz Hessian)
\end{enumerate}
for all $\vtheta, \vtheta' \in \Omega$. We write $\eta K \le T$ with $T = \cO(1)$.
\end{assumption}

At iteration $k$, the differences in parameters and loss between the two methods are
\[
\Delta\vtheta^k:=\vtheta_\cS^k-\vtheta_{\mathrm{JC}}^k,
\qquad
\Delta\cL^k:=\cL(\vtheta_\cS^k)-\cL(\vtheta_{\mathrm{JC}}^k).
\]
The subscript specifies the method and the superscript specifies the iteration.
For example, $\vg_\cS^k=\nabla\cL(\vtheta_\cS^k)$ and $\mH_\cS^k=\nabla^2\cL(\vtheta_\cS^k)$.
The matrix $\mM_\cS$ uses Hessian blocks evaluated at $\vtheta_\cS^k$, although we omit $k$ from its notation.
We set $L_\cS:=|\cS|L$, where $|\cS|$ is the number of pairs in $\cS$, so that $\|\mM_\cS\|\le L_\cS$.
The constants in the $\cO(\cdot)$ bounds may depend on $G,L,\rho,N,T$, but are independent of the step size $\eta$ and iteration index $k$.
The supporting bounds are given in Appendix~\ref{app:auxiliary_bounds}.

\subsubsection{Trajectory and loss recursions}

Comparing JC and scheme $\cS$ over multiple iterations requires accounting for the different parameter values they produce.
The comparison must include both the parameter difference from previous iterations and the effect of fresh block updates at the current iteration.
We bound $\Delta\vtheta^k$ to control approximation errors, then derive how this difference evolves and how it affects the loss difference.

\begin{lemma}[Size of the trajectory difference]
\label{lem:traj_size}
Under Assumption~\ref{ass:regularity}, with $\Delta\vtheta^0 = 0$ and $\eta k \le T$,
\begin{equation}
\|\Delta\vtheta^k\| \;\le\; \eta\,c_1\,\frac{e^{\eta k L} - 1}{L} \;\le\; \eta\,\frac{c_1(e^{TL} - 1)}{L} \;=\; \cO(\eta), \qquad c_1 := L_\cS G + \eta C_\varepsilon, \label{eq:traj_size}
\end{equation}
where $C_\varepsilon$ bounds the $\cO(\eta^3)$ remainder of Lemma~\ref{lem:staleness}.
\end{lemma}

The lemma shows that the parameters of JC and scheme $\cS$ remain within $\cO(\eta)$ of each other for all $k$ with $\eta k\le T$.
This bound allows the higher-order terms in the trajectory and loss to be bounded by $\cO(\eta^3)$.
The proof is given in Appendix~\ref{app:trajectory_size}.
We now derive the update relation for $\Delta\vtheta^k$.

\begin{lemma}[Trajectory Recursion for General Delay Patterns]
\label{lem:multistep_recursion}
Under Assumption~\ref{ass:regularity}, let $\cS$ be an acyclic delay pattern with associated matrix $\mM_\cS$. Suppose $\Delta\vtheta^0=0$ and $\eta K\le T=\cO(1)$. Then for $0 \le k < K$, the trajectory difference satisfies:

\begin{equation}
\Delta\vtheta^{k+1} = \underbrace{(\mI - \eta\mH_{\text{JC}}^k)\Delta\vtheta^k}_{\text{contract inherited}} + \underbrace{\eta^2 \mM_\cS \vg_\cS^k}_{\text{new forcing}} + \cO\!\left(\eta \|\Delta\vtheta^k\|^2 + \eta^3\right). \label{eq:traj_evolution}
\end{equation}
\end{lemma}

The first term propagates the existing trajectory difference through the linearized JC update, while the second introduces the effect of fresh block updates. By Lemma~\ref{lem:traj_size}, $\|\Delta\vtheta^k\|=\cO(\eta)$, so the displayed remainder is $\cO(\eta^3)$. The proof is given in Appendix~\ref{app:trajectory_recursion}.
Substituting this trajectory recursion into a second-order expansion of $\cL$ around $\vtheta_{\mathrm{JC}}^{k+1}$ yields the following loss recursion.

\begin{theorem}[Loss Evolution Recursion]
\label{thm:loss_recursion}
Under Assumption~\ref{ass:regularity}, let $\cS$ be an arbitrary acyclic delay pattern with associated staleness matrix $\mM_\cS$, and let $\eta K \le T = \cO(1)$ so that Lemma~\ref{lem:traj_size} applies. For any step $0 \le k < K$, define the trajectory difference $\Delta\vtheta^k := \vtheta_{\cS}^k - \vtheta_{\text{JC}}^k$ and loss difference $\Delta\cL^{k+1} := \cL_{\cS}^{k+1} - \cL_{\text{JC}}^{k+1}$ between the $\cS$-delayed method and Jacobi, where $\vg_{\text{JC}}^k := \nabla\cL(\vtheta_{\text{JC}}^k)$, $\mH_{\text{JC}}^k := \nabla^2\cL(\vtheta_{\text{JC}}^k)$, and $\vg_\cS^k := \nabla\cL(\vtheta_\cS^k)$. Then the loss difference at step $k+1$ satisfies:
\begin{equation}
\begin{aligned}
\Delta\cL^{k+1}
= \underbrace{\vg_{\text{JC}}^{k+1}{}^\top (\mI-\eta\mH_{\text{JC}}^k)\Delta\vtheta^k}_{\text{propagation: }\cO(\eta)}
 + \underbrace{\tfrac12(\Delta\vtheta^k)^\top\mH_{\text{JC}}^k\Delta\vtheta^k}_{\text{quadratic: }\cO(\eta^2)}
+ \underbrace{\eta^2\vg_{\text{JC}}^{k+1}{}^\top\mM_\cS\vg_\cS^k}_{\text{freshness: }\cO(\eta^2)}
 + \cO(\eta^3).
\end{aligned}
\label{eq:loss_evolution}
\end{equation}
\end{theorem}

The recursion separates the effect of previous updates from that of the current update scheme.
The first term carries the existing parameter difference through the linearized JC update and measures its effect on the loss; the second accounts for curvature along this difference; and the third captures the additional effect of fresh block values.
Since $\|\Delta\vtheta^k\|=\cO(\eta)$, the first term is $\cO(\eta)$ and the other two are $\cO(\eta^2)$.
When $\Delta\vtheta^k=0$, the loss difference reduces to $\eta^2\cC_\cS+\cO(\eta^3)$, recovering the common-start comparison in Theorem~\ref{thm:master}.
The proof is given in Appendix~\ref{app:loss_recursion}, using Lemma~\ref{lem:multistep_recursion} and a second-order Taylor expansion of the loss.

\subsection{Multi-Step Cumulative Loss Difference}

The preceding recursions describe how the parameter and loss differences change from one iteration to the next.
To compare JC and scheme $\cS$ after $K$ iterations, we must account for how differences introduced at earlier steps affect later updates, rather than simply adding the one-step comparisons from common starting points.
Starting from $\Delta\vtheta^0=0$, we repeatedly apply Lemma~\ref{lem:multistep_recursion} to collect these contributions to $\Delta\vtheta^K$, then substitute the resulting expression into the second-order loss expansion.
The following theorem gives the resulting loss comparison, with an accumulated error of $\cO(K\eta^3)$ under the stated assumptions.

\begin{theorem}[Signed Cumulative Terminal Comparison]
\label{thm:cumulative}
Under Assumption~\ref{ass:regularity}, let $\eta K\le T$ and
$\cS$ be acyclic. Define
\begin{equation}
\mPhi_{K,j+1}=\prod_{\ell=j+1}^{K-1}
(\mI-\eta\mH_{\mathrm{JC}}^\ell),\qquad
\vz_K=\eta^2\sum_{j=0}^{K-1}\mPhi_{K,j+1}\mM_\cS^j\vg_\cS^j,
\label{eq:signed_cumulative_state}
\end{equation}
where factors with larger indices appear to the left, and we set $\mPhi_{K,K}=\mI$ for $j=K-1$.
The matrix $\mM_\cS^j$ is evaluated at $\vtheta_\cS^j$. Then
\begin{equation}
\cL(\vtheta_\cS^K)-\cL(\vtheta_{\mathrm{JC}}^K)
=(\vg_{\mathrm{JC}}^K)^\top\vz_K
+\tfrac12\vz_K^\top\mH_{\mathrm{JC}}^K\vz_K
+\cO(K\eta^3).
\label{eq:signed_cumulative_loss}
\end{equation}
In particular, define
$b_K=\eta^2\sum_{j=0}^{K-1}(1+\eta L)^{K-1-j}
\|\mM_\cS^j\vg_\cS^j\|$. Then
\begin{equation}
|\Delta\cL^K|\le G b_K+\tfrac L2 b_K^2+\cO(K\eta^3).
\end{equation}
\end{theorem}

The matrix $\mPhi_{K,j+1}$ describes how subsequent updates change the parameter difference introduced at iteration $j$, and $\vz_K$ combines these contributions to approximate $\Delta\vtheta^K$.
The two terms in \eqref{eq:signed_cumulative_loss} give a signed estimate of the loss difference: a negative estimate favors $\cS$ and a positive estimate favors JC.
The proof is given in Appendix~\ref{app:cumulative}.

Together, these results extend the one-step cross-curvature comparison
to multiple iterations by accounting for both the accumulation and
propagation of update-induced parameter differences. The resulting
loss estimate connects local block interactions to the sign and
magnitude of the loss gap along the training trajectories. We next
evaluate how accurately the recursive and cumulative estimates capture
this gap.
\section{Experiments}
\label{sec:experiments}
\label{sec:literal-cumulative}

Our experiments evaluate whether the loss-comparison formulas explain
which update rule performs better after a given number of iterations. We ask
three questions: (i)~does the cumulative estimate identify the lower-loss
method and capture changes in the winner as $K$ varies; (ii)~does it
predict the magnitude of the terminal loss gap; and (iii)~how does
propagating an approximate parameter difference affect accuracy relative
to evaluating the loss recursion with the measured difference?
These questions assess the retained terms in
Theorems~\ref{thm:cumulative} and~\ref{thm:loss_recursion} against
independently measured JC and GS losses.

We address the same three questions in settings with different
block-coupling structures: layer parameters in a DNN, personal and
shared parameters in federated learning, and the two low-rank factors
in LoRA. In each setting, we compare the cumulative and recursive estimates
with the measured GS--JC loss differences over training, using the
same metrics for winner agreement and gap-estimation error. All three
studies use intermediate gradients and Hessians from the actual
training trajectories, providing direct diagnostics of the formulas.
We first specify the common evaluation protocol, then compare the
results across the three settings for each validation question.

\providecommand{\LiteralVisualDir}{figures}
\definecolor{LiteralInk}{HTML}{192D46}
\definecolor{LiteralViolet}{HTML}{7954C5}
\definecolor{LiteralTeal}{HTML}{008E9B}
\definecolor{LiteralHeader}{gray}{0.93}
\begin{table}[!htbp]
\centering
\begingroup\setlength{\abovecaptionskip}{0pt}\setlength{\belowcaptionskip}{6pt}
\caption{Winner identification and loss-gap accuracy of the cumulative
and recursive estimates at $K=500$ iterations.
Arrows indicate whether higher ($\uparrow$) or lower ($\downarrow$) values are better.
Bold denotes the best value for each metric within each task, including ties.}
\label{tab:literal-cumulative-winners}\endgroup
\begin{minipage}{\linewidth}
\centering
\fontsize{9}{11}\selectfont
\setlength{\tabcolsep}{1.5pt}
\renewcommand{\arraystretch}{1.22}
\begin{tabular*}{\linewidth}{@{\extracolsep{\fill}}llrrrrr@{}}
\toprule
\noalign{\begingroup\color{LiteralHeader}\hrule height\dimexpr\arraystretch\ht\strutbox+\arraystretch\dp\strutbox\relax\vskip-\dimexpr\arraystretch\ht\strutbox+\arraystretch\dp\strutbox\relax\endgroup}
\textbf{Task} & \textbf{Estimate} & \textbf{Correct} $\boldsymbol{\uparrow}$ & \textbf{JC hits} $\boldsymbol{\uparrow}$ & \textbf{GS hits} $\boldsymbol{\uparrow}$ & \textbf{MAE} $\boldsymbol{\downarrow}$ & \textbf{Max.\ AE} $\boldsymbol{\downarrow}$ \\
\midrule
\multirow{2}{*}{\textbf{DNN}} & \textcolor{LiteralViolet}{Cumulative} & 98.0\% & \textbf{45/45} & 445/455 & $5.58\!\times\!10^{-4}$ & $2.09\!\times\!10^{-3}$ \\
 & \textcolor{LiteralTeal}{Recursive} & \textbf{100.0\%} & \textbf{45/45} & \textbf{455/455} & $\boldsymbol{1.15\!\times\!10^{-4}}$ & $\boldsymbol{4.86\!\times\!10^{-4}}$ \\
\midrule
\multirow{2}{*}{\textbf{FL}} & \textcolor{LiteralViolet}{Cumulative} & 83.4\% & \textbf{25/28} & 392/472 & $2.95\!\times\!10^{-1}$ & $2.28\!\times\!10^{1}$ \\
 & \textcolor{LiteralTeal}{Recursive} & \textbf{93.2\%} & \textbf{25/28} & \textbf{441/472} & $\boldsymbol{1.79\!\times\!10^{-2}}$ & $\boldsymbol{3.86\!\times\!10^{-1}}$ \\
\midrule
\multirow{2}{*}{\textbf{LoRA}} & \textcolor{LiteralViolet}{Cumulative} & 97.6\% & \textbf{74/75} & 414/425 & $5.31\!\times\!10^{-2}$ & $4.54\!\times\!10^{0}$ \\
 & \textcolor{LiteralTeal}{Recursive} & \textbf{99.6\%} & \textbf{74/75} & \textbf{424/425} & $\boldsymbol{5.97\!\times\!10^{-4}}$ & $\boldsymbol{1.01\!\times\!10^{-1}}$ \\
\bottomrule
\end{tabular*}
\end{minipage}
\par\vspace{10pt}
\begin{minipage}{\linewidth}
\includegraphics[width=\linewidth]{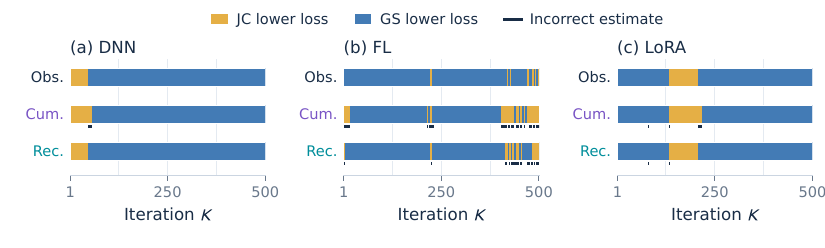}
\vspace{-1em}
\captionof{figure}{Measured winners (Obs.) and predictions from the cumulative
(Cum.) and recursive (Rec.) estimates.}
\label{fig:literal-winner-timeline}
\end{minipage}
\end{table}

\subsection{Evaluation Protocol and Theoretical Estimates}
\label{sec:literal-evaluation}

\paragraph{Controlled trajectories.}
For each task, JC and GS start from the same parameters and optimize
the same fixed training objective using full-batch block-gradient
descent with a constant learning rate, without momentum or adaptive
preconditioning. We evaluate both losses after every
complete block sweep $K=1,\ldots,500$. Gradients use the entire fixed
data subset, dropout is disabled, and batch-normalization running
statistics are fixed. There is no stochastic minibatch or client
sampling during training.
Appendix~\ref{app:intro_methods} describes the experimental details.

\paragraph{Recursive and cumulative estimates.}
We compare the measured GS--JC loss gap with the recursive and
cumulative estimates from Theorems~\ref{thm:loss_recursion}
and~\ref{thm:cumulative}, respectively. The recursive estimate uses
the actual parameter difference from the previous iteration.
The cumulative estimate instead propagates an approximate parameter
difference from the common initialization. Both use gradients and
Hessians evaluated along the actual JC and GS training trajectories.

\paragraph{Metrics and baselines.}
For a run of $K$ iterations, let $\widehat{\Delta\cL}^{k}$ be the
estimated gap at iteration $k$ and
$\Delta\cL^{k}=\cL(\vtheta_{\mathrm{GS}}^{k})-
\cL(\vtheta_{\mathrm{JC}}^{k})$ the measured gap, for $k=1,\ldots,K$.
Positive gaps favor JC and negative gaps favor GS. For both measured
and estimated gaps, define $w(x)=\mathrm{JC}$ when $x>0$,
and $w(x)=\mathrm{GS}$ when $x<0$. With $\mathbf 1\{\cdot\}$
denoting an indicator, the winner agreement is
$\mathrm{Correct}_K=\frac{100\%}{K}
\sum_{k=1}^{K}\mathbf 1\{w(\widehat{\Delta\cL}^{k})
=w(\Delta\cL^{k})\}$.
For $m\in\{\mathrm{JC},\mathrm{GS}\}$, define
$n_m=\sum_{k=1}^{K}\mathbf 1\{w(\Delta\cL^k)=m\}$
and $h_m=\sum_{k=1}^{K}\mathbf 1\{w(\Delta\cL^k)=m\}
\mathbf 1\{w(\widehat{\Delta\cL}^k)=m\}$.
The table displays \emph{JC hits} and \emph{GS hits} as the count
pairs $h_m/n_m$.
For gap magnitude, let
$e_k=|\widehat{\Delta\cL}^{k}-\Delta\cL^{k}|$ and report
$\mathrm{MAE}_K=\frac{1}{K}\sum_{k=1}^{K}e_k$ and
$\mathrm{Max.\ AE}_K=\max_{1\le k\le K}e_k$.
MAE measures average gap-estimation error, whereas Max.\ AE captures
the largest error. 


\subsection{Validation Across Block-Coupling Structures}
\label{sec:literal-results}

We apply the same evaluation protocol to three block structures.
For the DNN, we train an eight-layer ReLU MLP on 256 fixed MNIST
examples, with each affine layer forming a block and GS updating
from input to output.
For federated learning, we use a ResNet-18 on CIFAR-10
with three fully participating clients and eight fixed examples per
client. Both methods optimize the same mean client objective, with
GS updating all personal parameters before the shared parameters.
For LoRA, we adapt the query and value projections of T5-base on eight
fixed SST-2 examples using rank-eight adapters. The two blocks contain
all $B$ factors and all $A$ factors, with GS updating $B$ first.

\providecommand{\LiteralVisualDir}{figures}
\begin{figure}[t]
\centering
\includegraphics[width=\linewidth]{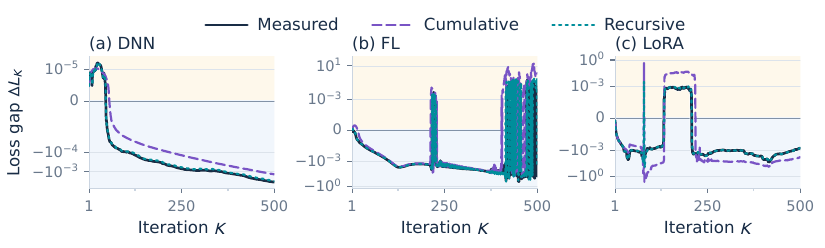}
\vspace{-1em}
\caption{Measured and estimated loss gaps,
$\Delta L_K=L_{\mathrm{GS}}^K-L_{\mathrm{JC}}^K$, over 500 iterations.
Positive/negative values favor JC/GS.}
\label{fig:literal-signed-loss-gaps}
\end{figure}

\paragraph{Identifying the lower-loss method.}
Table~\ref{tab:literal-cumulative-winners} reports overall winner
agreement and correct predictions for each method, while
Figure~\ref{fig:literal-winner-timeline} shows when errors occur.
For the DNN, JC has lower loss for the first 45 iterations and GS
thereafter. The cumulative estimate reproduces this ordering but
places the transition at $K=56$ instead of $K=46$. Its ten errors
are therefore confined to $K=46$--$55$, giving $98.0\%$ agreement;
the recursive estimate matches the measured winner at all 500 iterations.

For federated learning, GS has lower loss at 472 of the 500 iterations,
interrupted by brief intervals favoring JC
(Figure~\ref{fig:literal-winner-timeline}(b)). Both estimates correctly
identify 25 of the 28 JC wins. Their difference in agreement comes from the
GS intervals: the cumulative estimate incorrectly predicts JC at
80 such iterations, compared with 31 for the recursive estimate.
This accounts for the improvement in agreement from $83.4\%$ to
$93.2\%$, with most errors occurring late in training.

For LoRA, GS has lower loss at $K=1$--$132$, JC at
$K=133$--$207$, and GS again at $K=208$--$500$
(Figure~\ref{fig:literal-winner-timeline}(c)). The cumulative estimate
identifies 74 of the 75 JC wins and 414 of the 425 GS wins,
giving $97.6\%$ agreement. It places the sustained JC interval at
$K=134$--$217$: its errors are the one-step delay at $K=133$,
the ten-step delay in returning to GS at $K=208$--$217$, and an
isolated false JC prediction at $K=79$. The recursive estimate errors
only at $K=79$ and $K=133$, reaching $99.6\%$ agreement.

Overall, JC and GS each achieve lower loss at different stages of training in all three settings. Both the recursive and cumulative formulas capture the changes.
These results support the use of our theories to identify which method is favored along the observed training trajectories.

\paragraph{Estimating the loss-gap magnitude.}
The recursive estimate also has a lower mean absolute error (MAE) in all
three settings. The cumulative and recursive MAEs are, respectively,
$5.58\times10^{-4}$ and $1.15\times10^{-4}$ for the DNN,
$0.295$ and $0.0179$ for federated learning, and
$5.31\times10^{-2}$ and $5.97\times10^{-4}$ for LoRA.
The maximum absolute errors show the same pattern. For federated
learning, they are $22.8$ for the cumulative estimate and $0.386$
for the recursive estimate (Figure~\ref{fig:literal-signed-loss-gaps}(b)).
For LoRA, the corresponding values are $4.54$ and $1.01\times10^{-1}$
(Figure~\ref{fig:literal-signed-loss-gaps}(c)). 
Overall, the recursive formula estimates the loss gap more accurately than
the cumulative formula across all three settings, with lower mean and
maximum absolute errors.

\section{Conclusion and Future Work}
\label{sec:conclusion}

We developed a unified theory for comparing simultaneous, sequential,
and partially sequential block-gradient updates. Cross-block curvature
characterizes the leading one-step loss difference, while trajectory
and loss recursions extend the comparison to multiple iterations. Experiments on neural
networks, federated learning, and LoRA show that the resulting estimates
capture changes in the lower-loss method along observed training
trajectories. The recursive estimate, which uses measured parameter
differences, achieves higher winner agreement and lower loss-gap errors
than the cumulative estimate in all three settings.
Further directions include extending the analysis to
stochastic gradients and optimizers with momentum or adaptive step
sizes.

\subsection*{AI use statement}

In this work, we used generative AI tools to assist with writing and
language polishing, as well as literature retrieval and discovery of
related work. We have reviewed all AI-assisted work, including suggested
text and references, and made the final decisions on the manuscript's
wording and citations. We take responsibility for the final content of
this work, including text, claims, or artifacts produced with the aid of
generative AI.

\bibliographystyle{iclr2027_conference}
\bibliography{iclr2027_conference}

@book{young1971iterative,
  title={Iterative Solution of Large Linear Systems},
  author={Young, David M.},
  year={1971},
  publisher={Academic Press}
}

@book{saad2003iterative,
  title={Iterative Methods for Sparse Linear Systems},
  author={Saad, Yousef},
  edition={2nd},
  year={2003},
  publisher={SIAM}
}

@article{nesterov2012efficiency,
  title={Efficiency of coordinate descent methods on huge-scale optimization problems},
  author={Nesterov, Yurii},
  journal={SIAM Journal on Optimization},
  volume={22},
  number={2},
  pages={341--362},
  year={2012}
}

@article{attouch2013convergence,
  title={Convergence of descent methods for semi-algebraic and tame problems: proximal algorithms, forward-backward splitting, and regularized Gauss-Seidel methods},
  author={Attouch, H{\'e}dy and Bolte, J{\'e}r{\^o}me and Svaiter, Benar Fux},
  journal={Mathematical Programming},
  volume={137},
  number={1--2},
  pages={91--129},
  year={2013}
}

@article{beck2013convergence,
  title={On the convergence of alternating minimization for convex programming with applications to iteratively reweighted least squares and decomposition schemes},
  author={Beck, Amir},
  journal={SIAM Journal on Optimization},
  volume={25},
  number={1},
  pages={185--209},
  year={2015}
}

@article{wright2015coordinate,
  title={Coordinate descent algorithms},
  author={Wright, Stephen J.},
  journal={Mathematical Programming},
  volume={151},
  number={1},
  pages={3--34},
  year={2015}
}

@article{wang2019global,
  title={Global convergence of {ADMM} in nonconvex nonsmooth optimization},
  author={Wang, Yu and Yin, Wotao and Zeng, Jinshan},
  journal={Journal of Scientific Computing},
  volume={78},
  pages={29--63},
  year={2019}
}

@inproceedings{lian2015asynchronous,
  title={Asynchronous parallel stochastic gradient for nonconvex optimization},
  author={Lian, Xiangru and Huang, Yijun and Li, Yuncheng and Liu, Ji},
  booktitle={Advances in Neural Information Processing Systems (NeurIPS)},
  volume={28},
  year={2015}
}

@inproceedings{pillutla2022fedalt,
  title={Federated learning with partial model personalization},
  author={Pillutla, Krishna and Malik, Kshitiz and Mohamed, Abdel-rahman and Rabbat, Michael and Sanjabi, Maziar and Xiao, Lin},
  booktitle={Proceedings of the 39th International Conference on Machine Learning (ICML)},
  series={Proceedings of Machine Learning Research},
  volume={162},
  pages={17716--17758},
  year={2022}
}

@article{tseng2001convergence,
  title={Convergence of a block coordinate descent method for nondifferentiable minimization},
  author={Tseng, Paul},
  journal={Journal of Optimization Theory and Applications},
  volume={109},
  number={3},
  pages={475--494},
  year={2001}
}

@article{richtarik2014iteration,
  title={Iteration complexity of randomized block-coordinate descent methods for minimizing a composite function},
  author={Richt{\'a}rik, Peter and Tak{\'a}{\v{c}}, Martin},
  journal={Mathematical Programming},
  volume={144},
  number={1--2},
  pages={1--38},
  year={2014}
}

@article{razaviyayn2013unified,
  title={A unified convergence analysis of block successive minimization methods for nonsmooth optimization},
  author={Razaviyayn, Meisam and Hong, Mingyi and Luo, Zhi-Quan},
  journal={SIAM Journal on Optimization},
  volume={23},
  number={2},
  pages={1126--1153},
  year={2013}
}

@inproceedings{recht2011hogwild,
  title={Hogwild!: A lock-free approach to parallelizing stochastic gradient descent},
  author={Recht, Benjamin and R{\'e}, Christopher and Wright, Stephen and Niu, Feng},
  booktitle={Advances in Neural Information Processing Systems (NeurIPS)},
  volume={24},
  year={2011}
}

@inproceedings{rosca2021drift,
  title={Discretization drift in two-player games},
  author={Rosca, Mihaela and Wu, Yan and Dherin, Benoit and Barrett, David G. T.},
  booktitle={Proceedings of the 38th International Conference on Machine Learning (ICML)},
  series={Proceedings of Machine Learning Research},
  volume={139},
  pages={9064--9074},
  year={2021}
}

@inproceedings{lee2019wide,
  title={Wide neural networks of any depth evolve as linear models under gradient descent},
  author={Lee, Jaehoon and Xiao, Lechao and Schoenholz, Samuel S. and Bahri, Yasaman and Novak, Roman and Sohl-Dickstein, Jascha and Pennington, Jeffrey},
  booktitle={Advances in Neural Information Processing Systems (NeurIPS)},
  volume={32},
  year={2019}
}

@inproceedings{lecun1988framework,
  title={A theoretical framework for back-propagation},
  author={Le Cun, Yann},
  booktitle={Proceedings of the 1988 Connectionist Models Summer School},
  pages={21--28},
  year={1988},
  publisher={Morgan Kaufmann}
}

@inproceedings{mcmahan2017communication,
  title={Communication-efficient learning of deep networks from decentralized data},
  author={McMahan, Brendan and Moore, Eider and Ramage, Daniel and Hampson, Seth and Ag{\"u}era y Arcas, Blaise},
  booktitle={Proceedings of the 20th International Conference on Artificial Intelligence and Statistics (AISTATS)},
  series={Proceedings of Machine Learning Research},
  volume={54},
  pages={1273--1282},
  year={2017}
}

@inproceedings{hu2021lora,
  title={{LoRA}: Low-rank adaptation of large language models},
  author={Hu, Edward J. and Shen, Yelong and Wallis, Phillip and Allen-Zhu, Zeyuan and Li, Yuanzhi and Wang, Shean and Wang, Lu and Chen, Weizhu},
  booktitle={International Conference on Learning Representations (ICLR)},
  year={2022}
}

@inproceedings{hayou2024loraplus,
  title={{LoRA+}: Efficient low rank adaptation of large models},
  author={Hayou, Soufiane and Ghosh, Nikhil and Yu, Bin},
  booktitle={Proceedings of the 41st International Conference on Machine Learning (ICML)},
  series={Proceedings of Machine Learning Research},
  volume={235},
  pages={17783--17806},
  year={2024}
}

@inproceedings{zhu2026lorae2,
  title={{LoRA-E}$^2$: Effective and efficient low-rank adaptation},
  author={Zhu, Shengkun and Zeng, Jinshan and Wang, Yiming and Wang, Sheng and Sun, Yuan and Chen, Shangfeng and Yao, Yuan and Yang, Qiang},
  booktitle={Proceedings of the ACM Web Conference (WWW)},
  pages={7342--7353},
  year={2026}
}

@article{richtarik2016parallel,
  title={Parallel coordinate descent methods for big data optimization},
  author={Richt{\'a}rik, Peter and Tak{\'a}{\v{c}}, Martin},
  journal={Mathematical Programming},
  volume={156},
  number={1--2},
  pages={433--484},
  year={2016}
}

@article{xu2017globally,
  title={A globally convergent algorithm for nonconvex optimization based on block coordinate update},
  author={Xu, Yangyang and Yin, Wotao},
  journal={Journal of Scientific Computing},
  volume={72},
  number={2},
  pages={700--734},
  year={2017}
}

@article{kerahroodi2017coordinate,
  title={A coordinate-descent framework to design low {PSL/ISL} sequences},
  author={Alaee-Kerahroodi, Mohammad and Aubry, Augusto and De Maio, Antonio and Naghsh, Mohammad Mahdi and Modarres-Hashemi, Mahmoud},
  journal={IEEE Transactions on Signal Processing},
  volume={65},
  number={22},
  pages={5942--5956},
  year={2017}
}

@article{wang2008new,
  title={A new alternating minimization algorithm for total variation image reconstruction},
  author={Wang, Yilun and Yang, Junfeng and Yin, Wotao and Zhang, Yin},
  journal={SIAM Journal on Imaging Sciences},
  volume={1},
  number={3},
  pages={248--272},
  year={2008}
}

@article{xu2013block,
  title={A block coordinate descent method for regularized multiconvex optimization with applications to nonnegative tensor factorization and completion},
  author={Xu, Yangyang and Yin, Wotao},
  journal={SIAM Journal on Imaging Sciences},
  volume={6},
  number={3},
  pages={1758--1789},
  year={2013}
}

\clearpage
\appendix
\numberwithin{equation}{section}

\addtocontents{toc}{\protect\setcounter{tocdepth}{2}}
\begingroup
\renewcommand{\contentsname}{Contents of the Appendices}
\hypersetup{linktoc=all}
\tableofcontents
\endgroup
\clearpage

\section{Proofs for One-Step Comparisons}
\label{app:single_step}

This appendix proves Lemma~\ref{lem:staleness}, Theorem~\ref{thm:master},
and Corollary~\ref{cor:criterion}. We first derive the gradient
approximation under an acyclic delay pattern, then establish the
one-step loss expansion, and finally obtain the GS--JC criterion and
its implication for update ordering.

\subsection{Gradient approximation under an acyclic delay pattern}
\label{app:gradient_delay}

Acyclicity permits a sequential implementation respecting all fresh-value dependencies. Each already updated block has displacement $-\eta\vg_j+\cO(\eta^2)$, which is the input to the block-gradient expansion below.

\begin{proof}[Proof of Lemma~\ref{lem:staleness}]
By Taylor expansion: $\vtg_i = \vg_i + \sum_{j : (i,j) \in \cS} \mH_{ij} (\vtheta_j^{k+1} - \vtheta_j^k) + \cO(\|\Delta\vtheta\|^2)$. Since $\vtheta_j^{k+1} - \vtheta_j^k = -\eta \vg_j + \cO(\eta^2)$, the result follows.
\end{proof}

For full GS in the order $1,\ldots,N$, the same calculation gives the geometric interpretation used in the main text:
\begin{align*}
\vtg_i
&=\vg_i+\sum_{j<i}\mH_{ij}(\vtheta_j^{k+1}-\vtheta_j^k)+\cO(\eta^2)\\
&=\vg_i-\eta\sum_{j<i}\mH_{ij}\vg_j+\cO(\eta^2).
\end{align*}

\subsection{The one-step loss expansion}
\label{app:main_theorem}

\begin{proof}[Proof of Theorem~\ref{thm:master}]
For JC, the displacement is $\Delta\vtheta = -\eta \vg$. A second-order Taylor expansion gives
$\cL(\vtheta^k + \Delta\vtheta) = \cL^k + \vg^\top \Delta\vtheta + \frac{1}{2} \Delta\vtheta^\top \mH \Delta\vtheta + \cO(\|\Delta\vtheta\|^3).$
Substituting $\Delta\vtheta = -\eta \vg$ yields
\begin{align*}
\vg^\top \Delta\vtheta &= -\eta \|\vg\|^2, \\
\frac{1}{2} \Delta\vtheta^\top \mH \Delta\vtheta &= \frac{\eta^2}{2} \vg^\top \mH \vg.
\end{align*}
This yields the first equation.

For an acyclic pattern $\cS$, Lemma~\ref{lem:staleness} gives $\Delta\vtheta = -\eta \vtg^\cS = -\eta \vg + \eta^2 \mM_\cS \vg + \cO(\eta^3)$. Substitution into the same Taylor expansion yields
\begin{align*}
\vg^\top \Delta\vtheta &= \vg^\top (-\eta \vg + \eta^2 \mM_\cS \vg) + \cO(\eta^3) = -\eta \|\vg\|^2 + \eta^2 \vg^\top \mM_\cS \vg + \cO(\eta^3), \\
\frac{1}{2} \Delta\vtheta^\top \mH \Delta\vtheta &= \frac{1}{2} (-\eta \vg)^\top \mH (-\eta \vg) + \cO(\eta^3) = \frac{\eta^2}{2} \vg^\top \mH \vg + \cO(\eta^3).
\end{align*}
Combining these expressions gives the expansion for scheme $\cS$; subtracting the JC expansion proves the stated loss difference.
\end{proof}

\subsection{The GS--JC criterion and update ordering}
\label{app:gs_criterion}

\begin{proof}[Proof of Corollary~\ref{cor:criterion}]
For full GS, $\cS_{\mathrm{GS}}=\{(i,j):i>j\}$. Hessian symmetry gives
\begin{align}
\cC_{\cS_{\mathrm{GS}}}
&=\sum_{i>j}\vg_i^\top\mH_{ij}\vg_j
=\sum_{i<j}\vg_i^\top\mH_{ij}\vg_j
=\tfrac12\cC_{\mathrm{cross}}.
\label{eq:app_gs_mask}
\end{align}
Substitution into Theorem~\ref{thm:master} gives the stated loss difference. At a fixed common starting point with nonzero cross-curvature, the sign is determined by the leading term for sufficiently small $\eta$. If the cross-curvature is zero, the difference is $\cO(\eta^3)$.

More generally, a full GS sweep in any order includes exactly one orientation of each unordered pair $\{i,j\}$. Because
$\vg_i^\top\mH_{ij}\vg_j=\vg_j^\top\mH_{ji}\vg_i$,
its masked curvature is again $\cC_{\mathrm{cross}}/2$. Thus the losses from any two full GS orderings, started at the same point, differ by $\cO(\eta^3)$. This is a one-step statement.
\end{proof}

\section{Auxiliary Bounds and Trajectory Dynamics}
\label{app:trajectory}

The trajectory-size estimate precedes the Taylor expansion of the gradient difference: it supplies the uniform scale used to control the remainders. The proof of the trajectory recursion then records each residual explicitly.

\subsection{Auxiliary norm and Hessian bounds}
\label{app:auxiliary_bounds}

The following bounds are used throughout the multi-step proofs. They follow from Assumption~\ref{ass:regularity}. First, the mask $\mM_\cS$ inherits the Hessian bound,
\begin{equation}
\|\mM_\cS\| \;\le\; \sum_{(i,j) \in \cS} \|\mH_{ij}\| \;\le\; |\cS|\,L \;=:\; L_\cS, \qquad |\cS| \le \tfrac{1}{2}N(N-1) \text{ for acyclic } \cS, \label{eq:mask_norm}
\end{equation}
so $L_\cS$ depends only on $L$ and the block count. Second, consecutive Jacobi Hessians differ by $\cO(\eta)$: since $\|\vtheta_{\text{JC}}^{k+1} - \vtheta_{\text{JC}}^k\| = \eta\|\vg_{\text{JC}}^k\| \le \eta G$,
\begin{equation}
\|\mH_{\text{JC}}^{k+1} - \mH_{\text{JC}}^k\| \;\le\; \rho\,\eta\,G. \label{eq:hess_drift}
\end{equation}
Throughout, $\cO(\cdot)$ hides constants depending only on $G, L, \rho, N, T$ --- never on $\eta$ or $k$.

One combination recurs throughout: $\mI - \eta\mH_{\text{JC}}^k$ is the one-step Jacobi linearization applied to the inherited difference. By (A2) it satisfies
\begin{equation}
\left\|\mI - \eta\mH_{\text{JC}}^k\right\| \;\le\; 1 + \eta L. \label{eq:contraction_norm}
\end{equation}

\subsection{Size of the trajectory difference}
\label{app:trajectory_size}

\begin{proof}[Proof of Lemma~\ref{lem:traj_size}]
By Lemma~\ref{lem:staleness} the delayed displacement is $-\eta\vtg^\cS = -\eta\vg_\cS^k + \eta^2\mM_\cS\vg_\cS^k + \eta^3\bm{\varepsilon}^k$ with $\|\bm{\varepsilon}^k\| \le C_\varepsilon$, so subtracting the two update rules gives the \emph{exact} identity
\begin{equation}
\Delta\vtheta^{k+1} = \Delta\vtheta^k - \eta\left(\vg_\cS^k - \vg_{\text{JC}}^k\right) + \eta^2 \mM_\cS \vg_\cS^k + \eta^3\bm{\varepsilon}^k. \label{eq:traj_exact}
\end{equation}
By $L$-smoothness, $\|\vg_\cS^k - \vg_{\text{JC}}^k\| \le L\|\Delta\vtheta^k\|$. With \eqref{eq:mask_norm} and (A1) this yields the scalar recursion
\begin{equation*}
\|\Delta\vtheta^{k+1}\| \;\le\; (1 + \eta L)\,\|\Delta\vtheta^k\| + \eta^2\underbrace{\left(L_\cS G + \eta C_\varepsilon\right)}_{c_1}.
\end{equation*}
Iterating from $\Delta\vtheta^0 = 0$ gives the geometric sum $\|\Delta\vtheta^k\| \le \eta^2 c_1 \sum_{j=0}^{k-1}(1+\eta L)^j = \eta^2 c_1 \frac{(1+\eta L)^k - 1}{\eta L}$, and $(1+\eta L)^k \le e^{\eta k L} \le e^{TL}$ gives \eqref{eq:traj_size}.
\end{proof}

\subsection{The trajectory recursion}
\label{app:trajectory_recursion}

\begin{proof}[Proof of Lemma~\ref{lem:multistep_recursion}]

Taylor-expanding $\cL(\vtheta_\cS^k) = \cL(\vtheta_{\text{JC}}^k + \Delta\vtheta^k)$ about $\vtheta_{\text{JC}}^k$ to second order with Lagrange remainder gives, for some $\tilde\vtheta$ on the segment $[\vtheta_{\text{JC}}^k, \vtheta_\cS^k] \subset \Omega$,
\begin{equation}
\Delta\cL^k = \vg_{\text{JC}}^k{}^\top \Delta\vtheta^k + \frac{1}{2}(\Delta\vtheta^k)^\top \mH_{\text{JC}}^k \Delta\vtheta^k + \underbrace{\frac{1}{6}\nabla^3\cL(\tilde\vtheta)[\Delta\vtheta^k]^{\otimes 3}}_{=:\,R_3^k,\;\; |R_3^k| \,\le\, \frac{\rho}{6}\|\Delta\vtheta^k\|^3}. \label{eq:loss_from_traj}
\end{equation}
Assumption (A3) bounds the third derivative by $\rho$, giving the stated bound on $R_3^k$. For a quadratic loss, $R_3^k\equiv0$. We will also use \eqref{eq:loss_from_traj} at iteration $k+1$ to derive the loss recursion.

To derive the parameter recursion, Lemma~\ref{lem:staleness} gives the displacement under scheme $\cS$ as $-\eta\vtg^\cS = -\eta\vg_\cS^k + \eta^2\mM_\cS\vg_\cS^k + \eta^3\bm{\varepsilon}^k$ with $\|\bm{\varepsilon}^k\| \le C_\varepsilon$, so
\begin{equation*}
\vtheta_{\cS}^{k+1} = \vtheta_{\cS}^k - \eta \vg_{\cS}^k + \eta^2 \mM_\cS \vg_{\cS}^k + \eta^3\bm{\varepsilon}^k, \qquad\qquad \vtheta_{\text{JC}}^{k+1} = \vtheta_{\text{JC}}^k - \eta\vg_{\text{JC}}^k,
\end{equation*}
where the JC update is exact. Subtracting the two updates gives \eqref{eq:traj_exact}. A first-order expansion of the gradient under (A3) yields
\begin{equation}
\vg_\cS^k = \vg_{\text{JC}}^k + \mH_{\text{JC}}^k\Delta\vtheta^k + \vw^k, \qquad \|\vw^k\| \le \tfrac{\rho}{2}\|\Delta\vtheta^k\|^2, \label{eq:grad_diff_expansion}
\end{equation}
Substituting this identity into \eqref{eq:traj_exact} gives
\begin{equation}
\Delta\vtheta^{k+1} = (\mI - \eta\mH_{\text{JC}}^k)\Delta\vtheta^k + \eta^2 \mM_\cS \vg_\cS^k + \underbrace{\left(-\eta\vw^k + \eta^3\bm{\varepsilon}^k\right)}_{=:\,\vs^k}, \qquad \|\vs^k\| \le \tfrac{\eta\rho}{2}\|\Delta\vtheta^k\|^2 + \eta^3 C_\varepsilon, \label{eq:traj_compact}
\end{equation}
Equivalently,
\begin{equation}
\Delta\vtheta^{k+1} = \underbrace{(\mI - \eta\mH_{\text{JC}}^k)\Delta\vtheta^k}_{\text{contract inherited}} + \underbrace{\eta^2 \mM_\cS \vg_\cS^k}_{\text{new forcing}} + \cO\!\left(\eta \|\Delta\vtheta^k\|^2 + \eta^3\right). 
\end{equation}
The gradient $\vg_\cS^k$ is evaluated at the parameters of scheme $\cS$. Replacing it by $\vg_{\text{JC}}^k$ introduces the additional term $\eta^2\mM_\cS(\vg_\cS^k - \vg_{\text{JC}}^k) = \cO(\eta^2\|\Delta\vtheta^k\|)$, which is $\cO(\eta^3)$ by Lemma~\ref{lem:traj_size}. The two parts of the displayed remainder have distinct sources: $\eta\|\Delta\vtheta^k\|^2$ comes from the gradient expansion in \eqref{eq:grad_diff_expansion}, whereas $\eta^3$ comes from the block-update expansion in Lemma~\ref{lem:staleness}. Applying $\|\Delta\vtheta^k\|=\cO(\eta)$ then gives a uniform $\cO(\eta^3)$ bound.
\end{proof}

\section{Proof of the Loss Evolution Recursion}
\label{app:loss_recursion}

\begin{proof}[Proof of Theorem~\ref{thm:loss_recursion}]
Applying \eqref{eq:loss_from_traj} at iteration $k+1$ gives
\begin{equation}
\Delta\cL^{k+1} = \vg_{\text{JC}}^{k+1}{}^\top \Delta\vtheta^{k+1} + \frac{1}{2}(\Delta\vtheta^{k+1})^\top \mH_{\text{JC}}^{k+1} \Delta\vtheta^{k+1} + R_3^{k+1}, \label{eq:loss_k1_taylor}
\end{equation}
where $|R_3^{k+1}| \le \frac{\rho}{6}\|\Delta\vtheta^{k+1}\|^3$. To evaluate the linear term, we use the parameter recursion \eqref{eq:traj_compact},
\begin{equation}
\Delta\vtheta^{k+1} = (\mI - \eta\mH_{\text{JC}}^k)\Delta\vtheta^k + \eta^2 \mM_\cS \vg_\cS^k + \vs^k,
\end{equation}
where $\|\vs^k\| \le \tfrac{\eta\rho}{2}\|\Delta\vtheta^k\|^2 + \eta^3 C_\varepsilon$. Substituting into the linear term:
\begin{align}
\vg_{\text{JC}}^{k+1}{}^\top \Delta\vtheta^{k+1} &= \vg_{\text{JC}}^{k+1}{}^\top\!\left[(\mI - \eta\mH_{\text{JC}}^k)\Delta\vtheta^k + \eta^2\mM_\cS\vg_\cS^k + \vs^k\right] \notag \\
&= \underbrace{\vg_{\text{JC}}^{k+1}{}^\top(\mI - \eta\mH_{\text{JC}}^k)\Delta\vtheta^k}_{\text{propagation: } \cO(\eta)} + \underbrace{\eta^2\,\vg_{\text{JC}}^{k+1}{}^\top\mM_\cS\vg_\cS^k}_{\text{freshness: } \cO(\eta^2)} + \underbrace{\vg_{\text{JC}}^{k+1}{}^\top\vs^k}_{\cO(\eta\|\Delta\vtheta^k\|^2 + \eta^3)}. \label{eq:loss_linear_expanded}
\end{align}
The first two terms appear in \eqref{eq:loss_evolution}, while the remainder satisfies $|\vg_{\text{JC}}^{k+1}{}^\top\vs^k| \le G\,\|\vs^k\| = \cO(\eta\|\Delta\vtheta^k\|^2 + \eta^3)$ by (A1).
For the quadratic term, substituting \eqref{eq:traj_compact} gives
\begin{multline}
(\Delta\vtheta^{k+1})^\top\mH_{\text{JC}}^{k+1}\Delta\vtheta^{k+1} = \left[(\mI - \eta\mH_{\text{JC}}^k)\Delta\vtheta^k + \eta^2\mM_\cS\vg_\cS^k + \vs^k\right]^\top \mH_{\text{JC}}^{k+1} \\ \times \left[(\mI - \eta\mH_{\text{JC}}^k)\Delta\vtheta^k + \eta^2\mM_\cS\vg_\cS^k + \vs^k\right]. \label{eq:loss_quad_full}
\end{multline}

The contribution involving only the existing parameter difference is
\begin{align}
&\left[(\mI - \eta\mH_{\text{JC}}^k)\Delta\vtheta^k\right]^\top\mH_{\text{JC}}^{k+1}\left[(\mI - \eta\mH_{\text{JC}}^k)\Delta\vtheta^k\right] \notag \\
= &(\Delta\vtheta^k)^\top(\mI - \eta\mH_{\text{JC}}^k)^\top\mH_{\text{JC}}^{k+1}(\mI - \eta\mH_{\text{JC}}^k)\Delta\vtheta^k \notag \\
= &(\Delta\vtheta^k)^\top\left[\mH_{\text{JC}}^{k+1} - \eta(\mH_{\text{JC}}^k\mH_{\text{JC}}^{k+1} + \mH_{\text{JC}}^{k+1}\mH_{\text{JC}}^k) + \eta^2\mH_{\text{JC}}^k\mH_{\text{JC}}^{k+1}\mH_{\text{JC}}^k\right]\Delta\vtheta^k. \label{eq:quad_leading_expanded}
\end{align}

Hessian symmetry gives $(\Delta\vtheta^k)^\top\mH_{\text{JC}}^k\mH_{\text{JC}}^{k+1}\Delta\vtheta^k = (\Delta\vtheta^k)^\top\mH_{\text{JC}}^{k+1}\mH_{\text{JC}}^k\Delta\vtheta^k$. Combining this identity with $\mH_{\text{JC}}^{k+1} = \mH_{\text{JC}}^k + \cO(\eta)$ from \eqref{eq:hess_drift} yields
\begin{equation}
\bigl((\mI-\eta\mH_{\text{JC}}^k)\Delta\vtheta^k\bigr)^\top
\mH_{\text{JC}}^{k+1}\bigl((\mI-\eta\mH_{\text{JC}}^k)\Delta\vtheta^k\bigr) = (\Delta\vtheta^k)^\top\mH_{\text{JC}}^k\Delta\vtheta^k + \cO(\eta\|\Delta\vtheta^k\|^2). \label{eq:quad_leading_result}
\end{equation}

The remaining terms in \eqref{eq:loss_quad_full} are bounded using (A1), (A2), and \eqref{eq:traj_compact}.
The term quadratic in the fresh-update correction satisfies $\eta^4 (\mM_\cS\vg_\cS^k)^\top\mH_{\text{JC}}^{k+1}(\mM_\cS\vg_\cS^k) = \cO(\eta^4)$, and the term quadratic in the remainder satisfies $(\vs^k)^\top\mH_{\text{JC}}^{k+1}\vs^k = \cO(\|\vs^k\|^2) = \cO(\eta^2\|\Delta\vtheta^k\|^4 + \eta^6)$.
The interaction between the existing difference and the fresh-update correction satisfies $2\left[(\mI - \eta\mH_{\text{JC}}^k)\Delta\vtheta^k\right]^\top\mH_{\text{JC}}^{k+1}\left[\eta^2\mM_\cS\vg_\cS^k\right] = \cO(\eta^2\|\Delta\vtheta^k\|)$.
Their interactions with the remainder are bounded by $\cO(\eta\|\Delta\vtheta^k\|^3+\eta^3\|\Delta\vtheta^k\|)$ and $\cO(\eta^3\|\Delta\vtheta^k\|^2+\eta^5)$, respectively.
Combining these estimates gives
\begin{equation}
\tfrac{1}{2}(\Delta\vtheta^{k+1})^\top\mH_{\text{JC}}^{k+1}\Delta\vtheta^{k+1} = \tfrac{1}{2}(\Delta\vtheta^k)^\top\mH_{\text{JC}}^k\Delta\vtheta^k + \cO(\eta\|\Delta\vtheta^k\|^2 + \eta^2\|\Delta\vtheta^k\|+\eta^4). \label{eq:loss_quad_result}
\end{equation}

To bound the cubic remainder, \eqrefs{eq:traj_compact}{eq:contraction_norm} give $\|\Delta\vtheta^{k+1}\| = \cO(\|\Delta\vtheta^k\| + \eta^2)$. Hence
\begin{equation}
|R_3^{k+1}| \le \tfrac{\rho}{6}\|\Delta\vtheta^{k+1}\|^3 = \cO(\|\Delta\vtheta^k\|^3 + \eta^6). \label{eq:loss_cubic}
\end{equation}

Substituting \eqrefsiii{eq:loss_linear_expanded}{eq:loss_quad_result}{eq:loss_cubic} into \eqref{eq:loss_k1_taylor} yields
\begin{align}
\Delta\cL^{k+1} &= \vg_{\text{JC}}^{k+1}{}^\top(\mI - \eta\mH_{\text{JC}}^k)\Delta\vtheta^k + \eta^2\vg_{\text{JC}}^{k+1}{}^\top\mM_\cS\vg_\cS^k + \tfrac{1}{2}(\Delta\vtheta^k)^\top\mH_{\text{JC}}^k\Delta\vtheta^k \notag \\
&\quad + \underbrace{\cO(\eta\|\Delta\vtheta^k\|^2) + \cO(\eta^2\|\Delta\vtheta^k\|) + \cO(\|\Delta\vtheta^k\|^3) + \cO(\eta^3)}_{\text{total remainder}}.
\end{align}

By Lemma~\ref{lem:traj_size}, $\|\Delta\vtheta^k\| \le C\eta$ with $C = c_1(e^{TL}-1)/L$, so:
\begin{equation}
\eta\|\Delta\vtheta^k\|^2 \le C^2\eta^3, \quad \eta^2\|\Delta\vtheta^k\| \le C\eta^3, \quad \|\Delta\vtheta^k\|^3 \le C^3\eta^3.
\end{equation}

All remainder terms are therefore uniformly $\cO(\eta^3)$, which proves \eqref{eq:loss_evolution}.
\end{proof}

\section{Multi-Step Cumulative Derivation}
\label{app:cumulative}

\begin{proof}[Proof of Theorem~\ref{thm:cumulative}]
By Lemma~\ref{lem:traj_size}, $\|\Delta\vtheta^k\|=\cO(\eta)$
uniformly for $k\eta\le T$. Thus the local remainder in
\eqref{eq:traj_evolution} is uniformly $\cO(\eta^3)$. Unrolling that
recursion, with later factors on the left, yields
\begin{equation}
\Delta\vtheta^K=\eta^2\sum_{j=0}^{K-1}
\mPhi_{K,j+1}\mM_\cS^j\vg_\cS^j+\vr_K
=\vz_K+\vr_K,\qquad \|\vr_K\|\le C K\eta^3 e^{LT}.
\label{eq:traj_unrolled}
\end{equation}
Every propagation norm is at most
$(1+\eta L)^{K-1-j}\le e^{LT}$. The same estimate gives
$\|\vz_K\|\le b_K=\cO(K\eta^2)=\cO(\eta)$.
Substitute $\Delta\vtheta^K=\vz_K+\vr_K$ in
\eqref{eq:loss_from_traj}. The linear remainder is bounded by
$G\|\vr_K\|$; the quadratic difference is bounded by
$L\|\vz_K\|\|\vr_K\|+L\|\vr_K\|^2/2$; and the cubic Taylor
remainder is $\cO(\eta^3)$. For $K\ge1$ these terms are
$\cO(K\eta^3)$, proving the signed formula. Applying Cauchy--Schwarz
to its retained linear and quadratic terms proves the magnitude bound.
\end{proof}

\clearpage

\section{Experiments Details}
\label{app:intro_methods}

This appendix specifies the parameter blocks, update rules, and experimental
implementations. The three settings compare layer-wise updates in a neural network, personal/shared updates
in federated learning, and simultaneous/alternating factor updates in
LoRA. In each case, the distinction concerns the parameter values used
to evaluate a block gradient. We state the idealized update equations
and then document the optimizers, stochastic state, and loss-recording
conventions of the plotted runs. 
\begin{table}[htbp]
\centering
\setlength{\abovecaptionskip}{0pt}
\setlength{\belowcaptionskip}{6pt}
\caption{Block structure and update semantics of the experiments.}
\label{tab:intro_method_blocks}
\small
\setlength{\tabcolsep}{5pt}
\renewcommand{\arraystretch}{1.12}
\begin{tabularx}{\linewidth}{@{}
  >{\raggedright\arraybackslash}p{0.15\linewidth}
  >{\raggedright\arraybackslash}p{0.235\linewidth}
  *{2}{>{\raggedright\arraybackslash}X}@{}}
\toprule
\multirow{2}{*}{\textbf{Example}} &
\multirow{2}{*}{\textbf{Parameter blocks}} &
\multicolumn{2}{c@{}}{\textbf{Update semantics}} \\
\cmidrule(l){3-4}
& & \textbf{Simultaneous} & \textbf{Sequential} \\
\midrule
\textbf{MLP}\newline MNIST &
One weight--bias pair per layer &
\textbf{JC.} Compute all gradients before updating any parameters. &
\textbf{GS-like}\textsuperscript{$\dagger$}\textbf{.} Update output-to-input, using updated weights in the backward recursion. \\
\addlinespace[0.7em]
\textbf{ResNet-18}\newline CIFAR-10 &
Client-specific input layers $v_i$; shared output layers $u$ &
\textbf{FedSim.} Joint local updates of $(u_i,v_i)$. &
\textbf{FedAlt.} Personal phase, then shared phase using the updated personal weights. \\
\addlinespace[0.7em]
\textbf{LoRA}\newline SST-2 &
All $A$ factors;\newline all $B$ factors &
\textbf{JC.} One forward--backward pass at $(A^t,B^t)$. &
\textbf{GS.} Update $B$, then recompute the network to update $A$. \\
\bottomrule
\end{tabularx}
\par\vspace{4pt}
\begin{minipage}{\linewidth}
\footnotesize\raggedright
\textsuperscript{$\dagger$}\,The GS label in the MLP implementation denotes
in-place back-propagation, not a fully recomputed block gradient;
see Appendix~\ref{app:methods_mlp}.
\end{minipage}
\end{table}

\subsection{Layer-Wise Updates and Back-Propagation}
\label{app:methods_mlp}

Consider an $N$-layer feed-forward classifier with widths
$(d_0,\ldots,d_N)$ and $C=d_N$ output classes. Layer $\ell$ has
parameters $W_\ell\in\RR^{d_\ell\times d_{\ell-1}}$ and
$b_\ell\in\RR^{d_\ell}$, and the pair
$\theta_\ell=(W_\ell,b_\ell)$ forms one parameter block.
Given $h_0=x$, the network satisfies
\begin{equation}
z_\ell=W_\ell h_{\ell-1}+b_\ell,\qquad
h_\ell=\phi_\ell(z_\ell),\qquad \ell=1,\ldots,N.
\label{eq:methods_mlp_forward}
\end{equation}
For a minibatch $\mathcal B$ of $m$ examples and an output loss
$\ell(\cdot,y)$, the training objective is
\begin{equation}
L_{\mathcal B}(\theta)
=\frac{1}{m}\sum_{(x,y)\in\mathcal B}\ell(h_N(x;\theta),y).
\label{eq:methods_mlp_loss}
\end{equation}
Following \citet{lecun1988framework}, we introduce the layer states
as independent variables and write the equivalent constrained problem
\begin{equation}
\begin{aligned}
\min_{\theta,h}\quad
 &\frac{1}{m}\sum_{(x,y)\in\mathcal B}\ell(h_N(x),y),\\
\text{subject to}\quad
 &h_\ell(x)=\phi_\ell(W_\ell h_{\ell-1}(x)+b_\ell),
 \quad \ell=1,\ldots,N,\quad (x,y)\in\mathcal B.
\end{aligned}
\label{eq:methods_mlp_constrained}
\end{equation}
Here $h$ collects the states of all examples, with $h_0(x)=x$
fixed. Throughout a parameter sweep, the minibatch is held fixed.
We suppress the example index below and use $\sum_{\mathcal B}$
for summation over this minibatch.
Associating a multiplier $\lambda_\ell(x)\in\RR^{d_\ell}$ with
each constraint gives the Lagrange function
\begin{equation}
\mathcal J(\theta,h,\lambda)
=\frac{1}{m}\sum_{\mathcal B}
\left[\ell(h_N,y)+\sum_{\ell=1}^{N}
\lambda_\ell^\top\bigl(h_\ell-\phi_\ell(W_\ell h_{\ell-1}+b_\ell)\bigr)
\right].
\label{eq:methods_bp_lagrangian}
\end{equation}
At a differentiable constrained stationary point, the first-order
conditions are
\begin{equation}
\nabla_\lambda\mathcal J=0,\qquad
\nabla_h\mathcal J=0,\qquad
\nabla_\theta\mathcal J=0.
\label{eq:methods_bp_stationarity}
\end{equation}
For fixed $\theta$, the first condition recovers the forward
recursion in \eqref{eq:methods_mlp_forward}. The second determines
the multipliers through the terminal condition and backward recursion
\begin{equation}
\lambda_N=-\nabla_{h_N}\ell(h_N,y),\qquad
\lambda_\ell=W_{\ell+1}^\top
D\phi_{\ell+1}(z_{\ell+1})^\top\lambda_{\ell+1},
\quad \ell=N-1,\ldots,1,
\label{eq:methods_bp_multipliers}
\end{equation}
where $D\phi_\ell$ is the activation Jacobian. With the constraint
sign in \eqref{eq:methods_bp_lagrangian}, the preactivation error
$\delta_\ell=-D\phi_\ell(z_\ell)^\top\lambda_\ell$ satisfies
\begin{align}
\delta_N
 &=D\phi_N(z_N)^\top\nabla_{h_N}\ell(h_N,y),
\notag\\
\delta_\ell
 &=D\phi_\ell(z_\ell)^\top W_{\ell+1}^\top\delta_{\ell+1},
\quad \ell=N-1,\ldots,1.
\label{eq:methods_bp_adjoint}
\end{align}
Thus $\delta_\ell$ is the derivative of the per-example loss with
respect to $z_\ell$. For elementwise hidden activations, the second
line reduces to
$\delta_\ell=(W_{\ell+1}^\top\delta_{\ell+1})
\odot\phi_\ell'(z_\ell)$. In classification with one-hot labels
$y\in\RR^C$, softmax cross-entropy gives
\begin{equation}
\phi_N=\operatorname{softmax},\qquad
\ell(p,y)=-\sum_{c=1}^{C}y_c\log p_c,
\qquad \delta_N=h_N-y.
\label{eq:methods_mlp_cross_entropy}
\end{equation}

Let $h(\theta)$ and $\lambda(\theta)$ denote the resulting states
and multipliers. Since the constraints vanish and
$\nabla_h\mathcal J=0$, differentiating the reduced objective yields
\begin{equation}
\nabla_{\theta_\ell}L_{\mathcal B}(\theta)
=\left.\partial_{\theta_\ell}\mathcal J(\theta,h,\lambda)
\right|_{h=h(\theta),\,\lambda=\lambda(\theta)}.
\label{eq:methods_bp_reduced_gradient}
\end{equation}
In particular, the weight and bias gradients are
\begin{equation}
g_{W_\ell}(\theta)=\frac{1}{m}\sum_{\mathcal B}
\delta_\ell(\theta)h_{\ell-1}(\theta)^\top,
\qquad
g_{b_\ell}(\theta)=\frac{1}{m}\sum_{\mathcal B}\delta_\ell(\theta).
\label{eq:methods_bp_gradient}
\end{equation}
The remaining condition $\nabla_\theta\mathcal J=0$ requires
these gradients to vanish at stationarity. Gradient descent uses
their negatives as update directions. This derivation identifies
back-propagation as the computation of the reduced gradient through
forward states and backward adjoints.

The JC rule evaluates every layer gradient at the same parameter
vector $\theta^t$. Writing $h_j^t=h_j(\theta^t)$,
$z_j^t=z_j(\theta^t)$, and $\delta_j^t=\delta_j(\theta^t)$,
one forward pass and one backward pass give
\begin{align}
W_\ell^{t+1}
 &=W_\ell^t-\frac{\eta}{m}\sum_{\mathcal B}
   \delta_\ell^t(h_{\ell-1}^t)^\top,
&b_\ell^{t+1}
 &=b_\ell^t-\frac{\eta}{m}\sum_{\mathcal B}\delta_\ell^t,
\label{eq:methods_mlp_jc}
\end{align}
for $\ell=1,\ldots,N$ and step size $\eta>0$. Equivalently,
$\theta_\ell^{t+1}=\theta_\ell^t-\eta\nabla_{\theta_\ell}
L_{\mathcal B}(\theta^t)$. All states, activation Jacobians, and
weights in the adjoint recursion correspond to $\theta^t$.
This common evaluation state defines JC regardless of the order of
parameter assignment. In particular, the reverse order of derivative
evaluation in back-propagation does not imply GS optimization.

A GS sweep evaluates each layer gradient after incorporating the
preceding block updates. For the output-to-input order
$\ell=N,N-1,\ldots,1$, the parameter state before updating
layer $\ell$ is
\begin{equation}
\bar\theta^{t,\ell}
=\bigl(\theta_1^t,\ldots,\theta_\ell^t,
       \theta_{\ell+1}^{t+1},\ldots,\theta_N^{t+1}\bigr).
\label{eq:methods_mlp_gs_state}
\end{equation}
The corresponding partial derivative follows from the same
constrained formulation, with the forward and adjoint equations
solved at $\bar\theta^{t,\ell}$. All other parameter blocks are
held constant when taking this derivative, including those already
updated during the sweep.

Define $\bar h_j^{t,\ell}=h_j(\bar\theta^{t,\ell})$ and
$\bar z_j^{t,\ell}=z_j(\bar\theta^{t,\ell})$.
Because no layer up to $\ell$ has yet changed,
$\bar h_j^{t,\ell}=h_j^t$ for $0\leq j\leq\ell$ and
$\bar z_j^{t,\ell}=z_j^t$ for $1\leq j\leq\ell$.
The remaining forward states are obtained from
\begin{equation}
\bar z_j^{t,\ell}
=W_j^{t+1}\bar h_{j-1}^{t,\ell}+b_j^{t+1},\qquad
\bar h_j^{t,\ell}=\phi_j(\bar z_j^{t,\ell}),
\quad j=\ell+1,\ldots,N.
\label{eq:methods_mlp_gs_forward}
\end{equation}
The adjoint recursion at this state is
\begin{align}
\bar\delta_N^{t,\ell}
 &=D\phi_N(\bar z_N^{t,\ell})^\top
   \nabla_{h_N}\ell(\bar h_N^{t,\ell},y),
\notag\\
\bar\delta_j^{t,\ell}
 &=D\phi_j(\bar z_j^{t,\ell})^\top
   (W_{j+1}^{t+1})^\top\bar\delta_{j+1}^{t,\ell},
\quad j=N-1,\ldots,\ell.
\label{eq:methods_mlp_gs_adjoint}
\end{align}
For softmax cross-entropy,
$\bar\delta_N^{t,\ell}=\bar h_N^{t,\ell}-y$.
The layer update is therefore
\begin{align}
W_\ell^{t+1}
 &=W_\ell^t-\frac{\eta}{m}\sum_{\mathcal B}
   \bar\delta_\ell^{t,\ell}(h_{\ell-1}^t)^\top,
&b_\ell^{t+1}
 &=b_\ell^t-\frac{\eta}{m}\sum_{\mathcal B}\bar\delta_\ell^{t,\ell}.
\label{eq:methods_mlp_exact_gs}
\end{align}
Equivalently, $\theta_\ell^{t+1}=\theta_\ell^t
-\eta\nabla_{\theta_\ell}L_{\mathcal B}(\bar\theta^{t,\ell})$.
The feature $h_{\ell-1}^t$ can be reused because the parameters
producing it remain unchanged. The output and affected adjoints
must be recomputed at each substep. Hence the output-layer update
agrees with JC at a common starting point, while earlier-layer
updates generally differ through their dependence on the updated
subsequent layers.

These equations instantiate Definition~\ref{def:updates} with one
weight--bias pair per block and GS order
$(\theta_N,\ldots,\theta_1)$. A direct implementation of GS
repeats the forward and backward computations for each layer;
reusing unchanged intermediate states preserves the same gradients.
The loss-comparison results additionally require the stated
smoothness assumptions. For ReLU networks, these assumptions apply
locally away from activation boundaries and need not hold when an
update crosses a boundary.

For panel~(a), the available sequential implementation uses updated
weights in a single backward recursion while retaining the original
forward states and output residual. Its updates therefore need not
equal the mixed-state gradients above, which accounts for the
GS-like designation in Table~\ref{tab:intro_method_blocks}.
Both implementations use one forward pass and one backward recursion
per minibatch and clip each weight or bias gradient $g$ separately
using $\Pi_\tau(g)=g\min\{1,\tau/\|g\|_{\mathrm F}\}$ for
$g\ne0$, with $\Pi_\tau(0)=0$ and threshold $\tau>0$.

The plotting record identifies an 8-layer MNIST network with
learning rate $10^{-3}$. A matching configuration in the
available training script sets $(d_0,\ldots,d_N)=(784,512,384,256,192,128,64,32,10)$,
with ReLU hidden activations and a softmax output. It uses the
60,000 training images, flattened and divided by $255$,
Xavier-normal weights, zero biases, batch size $m=128$, $40$
epochs, and clipping threshold $\tau=5$.


\subsection{Federated Learning with Partial Model Personalization}
\label{app:methods_federated}

Following \citet{pillutla2022fedalt}, we partition the model into
shared parameters $u$ and client-specific parameters
$V=(v_1,\ldots,v_M)$. Client $i$ holds a dataset $\mathcal D_i$
of $n_i$ training examples, giving the objective
\begin{align}
\min_{u,v_1,\ldots,v_M} F(u,V)
&=\sum_{i=1}^{M}p_iF_i(u,v_i),\qquad p_i=\frac{n_i}{\sum_jn_j},
\notag\\
F_i(u,v_i)&=\frac{1}{n_i}\sum_{(x,y)\in\mathcal D_i}
\ell(f(x;u,v_i),y).
\label{eq:methods_fed_objective}
\end{align}
The shared parameters couple the client objectives, while each
$v_i$ is retained locally across communication rounds and is
excluded from server aggregation.

At communication round $t$, the server samples a client subset
$S_t$ and broadcasts $u^t$. Each selected client initializes
$u_{i,0}=u^t$ and $v_{i,0}=v_i^t$. FedSim uses the LocalSim
procedure to update both blocks from the same local state. With
step sizes $\gamma_v$ and $\gamma_u$ and stochastic partial
gradients evaluated on minibatch $\xi_{i,k}$, its plain-SGD form is
\begin{align}
v_{i,k+1}
 &=v_{i,k}-\gamma_v\widehat\nabla_vF_i(u_{i,k},v_{i,k};\xi_{i,k}),
\notag\\
u_{i,k+1}
 &=u_{i,k}-\gamma_u\widehat\nabla_uF_i(u_{i,k},v_{i,k};\xi_{i,k}),
\qquad k=0,\ldots,\tau-1.
\label{eq:methods_fedsim}
\end{align}
Both partial derivatives are evaluated at $(u_{i,k},v_{i,k})$,
so the shared update uses the personal parameters before their
update. After $\tau$ local steps, the client returns
$u_i^+=u_{i,\tau}$ and retains $v_i^+=v_{i,\tau}$.
The common gradient-evaluation state gives each local step its
JC structure.

FedAlt uses the LocalAlt procedure, which first adapts the personal
parameters with the broadcast shared parameters fixed:
\begin{equation}
v_{i,k+1}
=v_{i,k}-\gamma_v\widehat\nabla_vF_i(u^t,v_{i,k};\xi^v_{i,k}),
\quad k=0,\ldots,\tau_v-1,\qquad v_i^+=v_{i,\tau_v}.
\label{eq:methods_fedalt_v}
\end{equation}
The shared phase then starts from $u_{i,0}=u^t$ and holds the
adapted personal parameters $v_i^+$ fixed:
\begin{equation}
u_{i,k+1}
=u_{i,k}-\gamma_u\widehat\nabla_uF_i(u_{i,k},v_i^+;\xi^u_{i,k}),
\quad k=0,\ldots,\tau_u-1,\qquad u_i^+=u_{i,\tau_u}.
\label{eq:methods_fedalt_u}
\end{equation}
The shared gradient therefore depends on the updated personal
model, establishing the block order $(v_i,u_i)$. The two phases
may use different minibatches. With $\tau_v=\tau_u=1$, a common
step size, and exact gradients of the same fixed local objective,
these equations reduce to the two-block GS rule in
Definition~\ref{def:updates}. Multiple local steps, stochastic
sampling, and server aggregation introduce dynamics beyond a
single deterministic GS sweep.

After either local procedure, the implementation for panel~(b)
aggregates the participating clients' shared parameters using
normalized sample-count weights:
\begin{equation}
u^{t+1}=\sum_{i\in S_t}\frac{n_i}{\sum_{j\in S_t}n_j}\,u_i^+,
\qquad
v_i^{t+1}=
\begin{cases}
v_i^+,&i\in S_t,\\
v_i^t,&i\notin S_t.
\end{cases}
\label{eq:methods_fed_aggregation}
\end{equation}
Only shared state is communicated. These aggregation weights
differ from the uniform weights in the uniform-client formulation
of \citet{pillutla2022fedalt}. Under uniform client sampling and
unequal shard sizes, normalization over the selected subset need
not yield an unbiased estimate of the full-population weighted
aggregate.

Panel~(b) instantiates this partition in a CIFAR-resolution
ResNet-18 with a $3\times3$, stride-one, 64-channel stem and no
initial max-pooling layer. The personal block comprises the stem
and the first two residual stages, with 64 and 128 channels. The
shared block comprises the remaining stages, with 256 and 512
channels, global average pooling, and the ten-class classifier.
Each stage contains two basic residual blocks, giving the model
composition
$f(x;u,v_i)=f_{\mathrm{shared},u}(f_{\mathrm{personal},v_i}(x))$.

The 50,000 CIFAR-10 training examples are partitioned across $M=20$ clients by independently drawing class-wise proportions from a symmetric Dirichlet distribution with concentration $0.5$. Inputs are normalized with channel means $(0.4914,0.4822,0.4465)$ and standard deviations $(0.2470,0.2435,0.2616)$; no random crop or flip is applied.

Both methods run for 60 communication rounds, with ten clients
sampled uniformly without replacement per round. They use the
same initialization, data partition, and client-selection schedule,
with batch size 64, learning rate $10^{-3}$, momentum $0.9$,
weight decay $10^{-4}$, and seed 0. The implementation specifies
local computation in epochs: FedSim performs one epoch of joint
updates, while FedAlt performs one personal epoch followed by one
shared epoch. Each epoch traverses a newly shuffled local loader.

\subsection{Simultaneous and Alternating Updates for LoRA}
\label{app:methods_lora}

LoRA freezes a pretrained weight matrix
$W_0\in\RR^{d_{\mathrm{out}}\times d_{\mathrm{in}}}$ and
parameterizes its update through two low-rank factors $A$ and $B$
\citep{hu2021lora}. For input features $Z$, the layer output is
\begin{equation}
Y=(W_0+sBA)Z,\qquad
s=\frac{\alpha}{r},\qquad
A\in\RR^{r\times d_{\mathrm{in}}},\quad
B\in\RR^{d_{\mathrm{out}}\times r}.
\label{eq:methods_lora_layer}
\end{equation}
Here $r$ determines the adapter rank and $\alpha$ sets the scale
of the weight update. For a loss $L$ and output adjoint
$E=\partial L/\partial Y$, differentiation gives
\begin{equation}
\nabla_BL=sE(AZ)^\top,\qquad
\nabla_AL=sB^\top EZ^\top.
\label{eq:methods_lora_gradients}
\end{equation}
The features and adjoints in these expressions are evaluated at
the current network state. With adapter dropout, the adapter branch
uses the masked features in place of $Z$, while the pretrained
branch retains the original input. For a network with multiple
adapters, we use $A$ and $B$ to denote the collections of all
corresponding factors, forming two parameter blocks across the
network.

The coupling between these blocks permits simultaneous and
alternating updates, as studied by \citet{zhu2026lorae2}. For a
fixed minibatch objective $L_{\mathcal B_t}$, the simultaneous
plain-gradient rule evaluates both partial derivatives at
$(A^t,B^t)$:
\begin{equation}
B^{t+1}=B^t-\eta_B\nabla_BL_{\mathcal B_t}(A^t,B^t),\qquad
A^{t+1}=A^t-\eta_A\nabla_AL_{\mathcal B_t}(A^t,B^t).
\label{eq:methods_lora_jc}
\end{equation}
A single forward--backward pass supplies both gradients. The
alternating rule first updates $B$ and then evaluates the gradient
for $A$ at $(A^t,B^{t+1})$:
\begin{equation}
B^{t+1}=B^t-\eta_B\nabla_BL_{\mathcal B_t}(A^t,B^t),\qquad
A^{t+1}=A^t-\eta_A\nabla_AL_{\mathcal B_t}(A^t,B^{t+1}).
\label{eq:methods_lora_gs}
\end{equation}
Computing the second gradient requires a new network evaluation
after the $B$ update. The resulting adjoints, and any adapter input
features affected by that update, must correspond to the mixed
state $(A^t,B^{t+1})$. Substituting $B^{t+1}$ into
\eqref{eq:methods_lora_gradients} while retaining the previous
adjoints does not generally recover this partial derivative.

With a common step size $\eta_A=\eta_B=\eta$ and a fixed
deterministic objective, these rules are exactly the two-block JC and
GS updates in Definition~\ref{def:updates}, with block order $(B,A)$.
The cross-Hessian blocks between the factors determine the coupling
term in the one-step comparison of Theorem~\ref{thm:master}, while
the evolution of the gradients and curvature along the two
trajectories enters Theorem~\ref{thm:cumulative}.

Panel~(c) instantiates these updates in T5-base, as identified by
the saved adapter configurations and training records. SST-2 is
formulated as text-to-text classification, with the input prefix
\texttt{sentiment classification:} and target strings
\texttt{negative} and \texttt{positive}. Inputs are truncated or
padded to 128 tokens and targets to two tokens. 

The run uses the first 2,000 training examples and first 500 validation examples in the stored dataset order; the training loader subsequently shuffles examples. Adapters are attached to the query and value projections in the encoder self-attention, decoder self-attention, and decoder cross-attention modules. Their settings are $r=8$, $\alpha=16$, dropout $0.1$, no trainable biases, and 884,736 trainable adapter parameters. All pretrained parameters remain frozen.
Both runs use the same Gaussian initialization $A_{ij}^0\sim\mathcal N(0,2/d_{\mathrm{in}})$ and $B^0=0$.
They retain their respective update rules throughout training: simultaneous factor updates for JC and $B$-then-$A$ updates for GS.

\end{document}